\documentclass{article}
\usepackage{PRIMEarxiv}
\usepackage[utf8]{inputenc}
\usepackage[T1]{fontenc}
\usepackage{microtype}
\usepackage{fontawesome5}
\usepackage[numbers,square,sort]{natbib}
\setcitestyle{citesep={,}}
\makeatletter
\AtBeginDocument{\let\NAT@space\@empty}
\makeatother
\usepackage{fancyhdr}
\date{}\makeatletter\renewcommand{\@toptitlebar}{\hrule height 3pt\vskip 0.25in\vskip -\parskip}\renewcommand{\@bottomtitlebar}{\vskip 0.29in\vskip -\parskip\hrule height 0.5pt\vskip 0.09in}\makeatother

\usepackage{amsmath,amsfonts,bm}

\def\eqref#1{equation~\ref{#1}}

\def\1{\bm{1}}

\DeclareMathAlphabet{\mathsfit}{\encodingdefault}{\sfdefault}{m}{sl}
\SetMathAlphabet{\mathsfit}{bold}{\encodingdefault}{\sfdefault}{bx}{n}

\newcommand{\E}{\mathbb{E}}

\usepackage{amsmath,amssymb,amsthm}
\usepackage{booktabs}
\usepackage{colortbl}
\usepackage{arydshln}
\usepackage{flafter}
\usepackage[section]{placeins}
\usepackage{needspace}
\newcommand{\finishwrap}{%
  \par
  \ifnum\value{WF@wrappedlines}>1
    \vspace{\dimexpr\value{WF@wrappedlines}\baselineskip-\baselineskip\relax}%
  \fi
  \WFclear
}
\usepackage{graphicx}\usepackage{tikz}\usetikzlibrary{positioning,arrows.meta}\usepackage{wrapfig}\usepackage{capt-of}
\usepackage{xcolor}
\definecolor{citationblue}{HTML}{006D9A}
\usepackage[colorlinks=true,linkcolor=black,citecolor=citationblue,urlcolor=citationblue]{hyperref}
\usepackage{topconf-callouts}
\usepackage{url}
\usepackage{algorithm}
\usepackage{algpseudocode}

\newtheorem{lemma}{Lemma}
\newtheorem{proposition}{Proposition}

\title{DiffusionART: Efficient, Stable and Scalable Reinforcement of Diffusion Models}

\author{%
\begin{minipage}{0.98\textwidth}\centering\normalfont\normalsize
{\bfseries Zhenglin Wan$^{1,2}$ \quad Jiayu Chen$^{2,\dagger}$ \quad
Zhengfeiyang Zhang$^{1}$ \quad Pengfei Zhou$^{1}$ \quad Wangbo Zhao$^{3}$}\\[0.25em]
{\bfseries Fangyuan Kong$^{2}$ \quad Pengfei Wan$^{2}$ \quad
Kun Gai$^{2}$ \quad Yang You$^{1}$}\\[0.4em]
$^{1}$HPC-AI Lab, National University of Singapore \quad
$^{2}$Kling AI \quad $^{3}$HKUST\\[0.3em]
{\small \textcolor{black}{\faEnvelope}\ Email: \href{mailto:vanzl@u.nus.edu}{\nolinkurl{vanzl@u.nus.edu}}, \href{mailto:jia768167535@gmail.com}{\nolinkurl{jia768167535@gmail.com}}}\\[0.1em]{\small\faGithub\ GitHub: \url{https://github.com/vanzll/DiffusionART}}\\[0.1em]
{\small\textcolor[HTML]{F97316}{\faFire} Project Page: \url{https://vanzll.github.io/diffusionart-project/}}
\end{minipage}%
}

\newcommand{\method}{DiffusionART}
\newcommand{\rms}{\operatorname{RMS}}

\begin{document}
\addtocontents{toc}{\protect\setcounter{tocdepth}{-1}}
\pagestyle{fancy}
\raggedbottom

\maketitle\vspace{-\baselineskip}
\begingroup
\renewcommand{\thefootnote}{\ensuremath{\dagger}}
\footnotetext[1]{Corresponding author.}
\endgroup

\begin{abstract}
Online reinforcement learning (RL) fine-tunes diffusion models with black-box rewards, but reverse-process methods such as Flow-GRPO are limited in efficiency by poor credit assignment, while forward-process methods such as DiffusionNFT are efficient yet unstable without carefully tuned exponential-moving-average (EMA) off-policy training. We show that NFT-style objectives contain \textcolor[HTML]{007F88}{\textbf{a hidden forward--reverse consistency force}} that joint positive--negative training cancels only partially, and that EMA stabilizes training through \textcolor[HTML]{007F88}{\textbf{restoration}} rather than off-policy rollouts. We introduce \textcolor[HTML]{D55E00}{\textbf{\method{}}} (\underline{Diffusion} Reinforcement with \underline{A}ggregation-First \underline{R}eward \underline{T}ilting), which aggregates each prompt's rewards into a shared reward-improving direction for reverse-process states, removing the consistency force by construction, and adds a trust-region-masked second update for drift-adaptive restoration without EMA tuning. It requires no differentiable reward model, remains effective with fewer samples per prompt, and supports pairwise reward models (DiffusionART-Pairwise). Across image and video generation, \method{} trains efficiently and stably without per-task tuning. In particular, on HunyuanVideo-13B with equally weighted visual and motion quality rewards, DiffusionART reaches the late-stage mean reward levels of DiffusionNFT and Flow-GRPO-Fast with $12.02\times$ and $11.72\times$ fewer GPU hours, respectively, while achieving 28.8\% and 53.0\% higher mean training rewards.
\end{abstract}

\begin{figure*}[h!]
\centering

\includegraphics[width=0.8\textwidth]{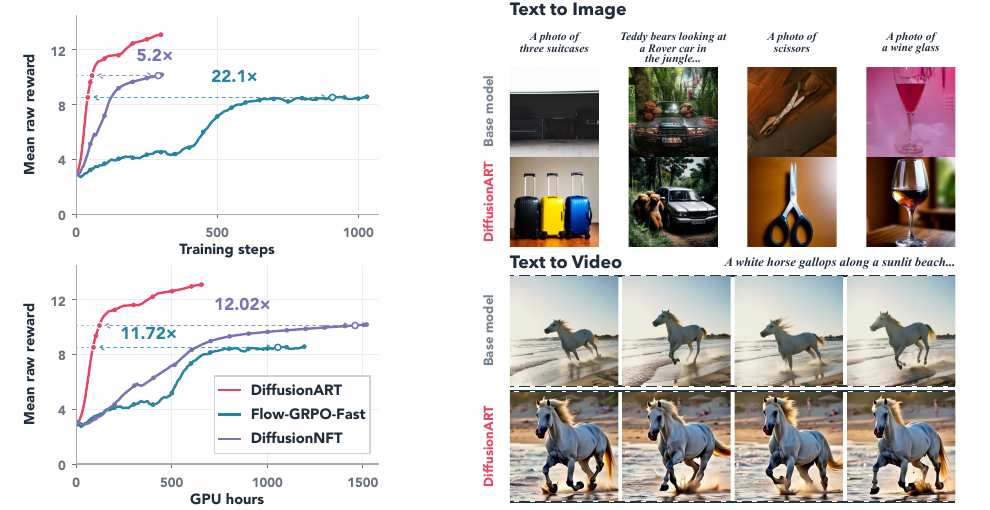}
\caption{\textbf{Efficient learning and improved generations.} \textbf{Left:} Mean of raw VQ and MQ (equal weights). \textbf{Right:} Selected same-prompt images and video frames before and after DiffusionART.}
\label{fig:teaser}
\end{figure*}

\section{Introduction}
\label{sec:introduction}

Visual diffusion and flow-matching models\footnote{We use ``diffusion'' as a general term. Our paper targets rectified-flow with linear data--noise interpolation~\citep{lipman2022flowmatching}, which shares the core principles of diffusion.} are typically trained to fit the data distribution by regressing noise or velocity targets rather than to optimize downstream criteria such as visual quality, motion quality, or human preference, making reward-guided fine-tuning central to recent image and video generation tasks~\citep{xu2026qwenimage2rl,gao2025seedance}. We focus on online diffusion Reinforcement Learning (RL), which, unlike differentiable reward-model optimization such as ReFL~\citep{xu2023imagereward} or offline preference learning such as Diffusion-DPO~\citep{wallace2023diffusiondpo}, evaluates clean samples from the evolving generator and can use black-box rewards. Because these rewards arrive only after a complete output is generated, online diffusion RL faces a credit-assignment problem: costly output-level feedback must be translated into effective updates across many denoising states~\citep{liu2025flowgrpo,zheng2025diffusionnft}.

Within this paradigm, \emph{reverse-process RL}~\citep{black2023ddpo}, such as Flow-GRPO and its variants~\citep{liu2025flowgrpo,xue2025dancegrpo,he2025tempflowgrpo,wang2026grpoguard}, optimizes sampled denoising transitions via group-based endpoint reward feedback. It trains on states actually visited by the generator but requires ODE-to-SDE conversion for tractable likelihoods, and its policy gradient updates each selected transition along its own transition-level sampling noise, weighted by a trajectory-level advantage that this transition only partly determines, resulting in poor credit assignment and limiting training efficiency. \emph{Forward-process RL}, such as DiffusionNFT~\citep{zheng2025diffusionnft}, instead applies forward noising to the generated clean samples and regresses velocities with reward-weighted positive and negative objectives. It needs no transition likelihoods and is highly efficient, but it is unstable without two stabilizers: joint positive--negative training and an exponential moving average (EMA) rollout policy whose lag must be tuned for each task~\citep{zheng2025diffusionnft}. Neither family is thus both efficient and stable without per-task tuning. In this paper, we focus on studying the instability of forward-process RL and developing a scalable algorithm which is both stable and efficient.

We first examine NFT-style training, starting from a puzzle: its positive and negative branches share the same ideal reward-improving direction, so either alone should suffice, yet each alone collapses (Figure~\ref{fig:nft_branch_balance}(a)). We find that NFT's targets, built from forward-noised endpoints, implicitly assume that the model's sampling velocity equals the ideal forward velocity induced by its own endpoints. The two differ in practice, and the gap introduces a hidden non-reward-improving \emph{forward--reverse consistency force}: the positive branch pulls each state's velocity toward the ideal forward velocity and the negative branch pushes it away. Joint training can cancel this force but only partially, since reward centering balances rewards across samples rather than at each state. The residual can destabilize training and slow the reward optimization. We further show that EMA stabilizes NFT through the \emph{restoration force} of its lagged reference, not through off-policy rollouts, but the EMA trick must be tuned per task.

Driven by these insights, we introduce \textbf{DiffusionART} (\underline{Diffusion} Reinforcement with \underline{A}ggregation-First \underline{R}eward \underline{T}ilting). It borrows supervised velocity regression from forward-process training but applies updates to states and actions during reverse-process rollouts. It first aggregates each prompt's samples and rewards into a shared reward-improving direction $G(c)$, then regresses each visited state's action toward its guided target (Section~\ref{sec:art_guidance}). As this target never involves the forward velocity, the consistency force is removed by construction rather than by cross-sample cancellation, and each state receives an explicit reward-improving direction instead of a scalar advantage passed through likelihood ratios. A second update with a trust-region mask mechanism on each rollout adds a restoration force proportional to each state's drift, replacing EMA with no schedule to tune (Section~\ref{sec:art_restoration}). We summarize our main contributions as follows:

\TopConfRenderCallout{TopConfInsightBlue}{1.5mm}
{\textbf{Main Contributions}\\[3pt]}
{
\noindent\textbf{1. Diagnosing the instability of NFT-style training.} The NFT objective implicitly contains a harmful forward--reverse consistency force that joint training cancels only partially (Section~\ref{sec:nft_failure_modes}).\par\smallskip
\noindent\textbf{2. Identifying restoration as the key stabilizer.} EMA stabilizes NFT through restoration rather than off-policy data, but requires per-task tuning   (Section~\ref{sec:ema_restoration}).\par\smallskip
\noindent\textbf{3. Aggregation-first reward tilting with adaptive restoration.} DiffusionART removes the consistency force by construction, eliminating the need for EMA with drift-adaptive restoration (Section~\ref{sec:method}).  \par\smallskip
\noindent\textbf{4. Efficient and stable training.}
DiffusionART matches the strongest baseline's efficiency on SD3.5-M without EMA or per-task tuning.
Under VQ:MQ $=1{:}1$ weights on HunyuanVideo-13B, it reaches the late-stage mean VQ and MQ rewards of DiffusionNFT and Flow-GRPO-Fast with $12.02\times$ and $11.72\times$ fewer GPU hours, while achieving 28.8\% and 53.0\% higher mean training rewards, respectively.
For image-to-video generation, it reaches MQ 6 with $4.8\times$ fewer GPU hours than Flow-GRPO-Fast (Section~\ref{sec:experiments}).

}

\section{Preliminaries}
\label{sec:preliminaries}

Diffusion RL~\citep{xu2026pathspace} includes reverse-process policy optimization, such as Flow-GRPO~\citep{liu2025flowgrpo}, and forward-process optimization, such as DiffusionNFT~\citep{zheng2025diffusionnft,zhu2026crd,xue2026advantageweightedmatchingaligning}. We focus on DiffusionNFT and defer Flow-GRPO to Appendix~\ref{app:flowgrpo_objective}. At each iteration, for each prompt $c$, DiffusionNFT samples $K$ clean images or videos $x_0$ from a frozen rollout policy $p_{\mathrm{old}}(x_0\mid c)$; we call these generated samples \emph{endpoints}. DiffusionNFT maps each raw reward $R_i:=R(x_0^{(i)},c)$ to a \emph{goodness} weight $r_i:=r(x_0^{(i)},c)=\tfrac12[1+\operatorname{clip}((R_i-\bar R_c)/Z_c,-1,1)]\in[0,1]$, where $\bar R_c=K^{-1}\sum_iR_i$ is the prompt's mean reward and $Z_c>0$ controls the scale; without clipping, $\sum_i(2r_i-1)=0$. Weighting $p_{\mathrm{old}}$ by $r$ and $1-r$ induces the positive and negative endpoint distributions $p_+(x_0\mid c)\propto r(x_0,c)p_{\mathrm{old}}(x_0\mid c)$ and $p_-(x_0\mid c)\propto[1-r(x_0,c)]p_{\mathrm{old}}(x_0\mid c)$. DiffusionNFT is motivated by the result that the ideal velocities induced by forward noising the positive, original, and negative endpoint distributions are collinear. Moving toward the positive field and away from the negative field therefore follows the same ideal improvement direction.

We parameterize generation from $t=1$ (noise) to $t=0$ (clean samples), so velocity points along the denoising direction. For a forward state $x_t=\alpha_t x_0+\sigma_t\epsilon$, let $u_t=-\dot{\alpha}_t x_0-\dot{\sigma}_t\epsilon$ denote its sample velocity target, where dots denote derivatives with respect to $t$; standard rectified-flow interpolation~\citep{lipman2022flowmatching,liu2022flowstraightfastlearning} gives $u_t=x_0-\epsilon$. Let $v_{\mathrm{old}}:=v_{\theta_{\mathrm{old}}}(x_t,t,c)$ and $v_\theta(x_t,t,c)$ be the velocities predicted by the rollout policy $p_{\mathrm{old}}$ (frozen within an iteration) and the trainable policy, respectively, and let $\mathcal D_{\mathrm{old}}$ be the training distribution of $(c,x_0,r,t,\epsilon)$ with $x_0\sim p_{\mathrm{old}}(\cdot\mid c)$ and $\epsilon\sim\mathcal N(0,I)$. For $\beta>0$, DiffusionNFT constructs $v_{\theta,\pm}=v_{\mathrm{old}}\pm\beta(v_\theta-v_{\mathrm{old}})$ and minimizes $\mathcal L_{\mathrm{NFT}}=\mathbb E_{\mathcal D_{\mathrm{old}}}[r\lVert v_{\theta,+}-u_t\rVert_2^2+(1-r)\lVert v_{\theta,-}-u_t\rVert_2^2]$. Equivalently, let $a=(2r-1)/\beta$ and define the frozen target $v_{\mathrm{tgt}}=a u_t+(1-a)v_{\mathrm{old}}$. Then
\begin{equation}
\label{eq:nft_objective}
\mathcal L_{\mathrm{NFT}}=\beta^2\mathbb E_{\mathcal D_{\mathrm{old}}}\!\left[\lVert v_\theta-v_{\mathrm{tgt}}\rVert_2^2\right]+\mathrm{const},
\end{equation}
where $\mathrm{const}$ is independent of $\theta$. Notably, NFT retains a $v_{\mathrm{old}}$ anchor in $v_{\mathrm{tgt}}$. At the on-policy initialization $v_\theta=v_{\mathrm{old}}$, each sample's output-space descent direction is proportional to $(2r-1)(u_t-v_{\mathrm{old}})$, so high- and low-reward samples drive the same regression residual in opposite directions (see Appendix~\ref{app:nft_mse_equivalence} for both derivations). After optimizing the rollout batch, DiffusionNFT updates the rollout policy by exponential moving average (EMA): $\theta_{\mathrm{old}}\leftarrow\eta\theta_{\mathrm{old}}+(1-\eta)\theta$.

\section{Understanding NFT-style Diffusion Reinforcement Learning}
\label{sec:motivation}
DiffusionNFT is highly efficient, but its stability relies on joint positive--negative training and an off-policy EMA rollout policy~\citep{zheng2025diffusionnft}. We examine why both are needed and where they fall short: joint training only partially cancels a hidden forward--reverse consistency force, while EMA provides restoration whose strength requires task-specific tuning. These limitations motivate the two components of DiffusionART in Section~\ref{sec:method}.

\subsection{Joint Positive--Negative Training: Partial Cancellation of a Hidden Consistency Force}

\label{sec:nft_failure_modes}

\begin{figure*}[t]
\centering

\makebox[\textwidth][c]{
\includegraphics[width=0.8\textwidth]{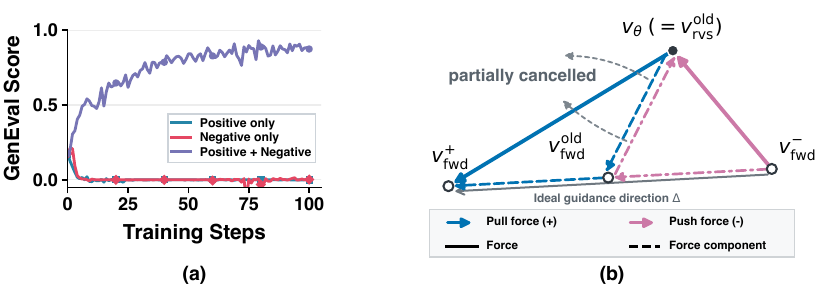}
}

\caption{(a) Ablation results for positive and negative objectives. (b) The positive and negative objectives partially cancel the consistency force.}
\label{fig:nft_branch_balance}
\end{figure*}

We first remove EMA to isolate joint training, so $v_\theta=v_{\mathrm{old}}$ at every on-policy update. Consistent with \citet{zheng2025diffusionnft}, either branch alone collapses, whereas joint training learns effectively (Figure~\ref{fig:nft_branch_balance}(a)). This is puzzling: the ideal positive and negative guidance directions are collinear, suggesting that either branch alone should suffice.

To make this precise, for $b\in\{+,\mathrm{old},-\}$, let $P_b$ denote the endpoint distribution $p_b(x_0\mid c)$, and let $Q_b$ denote the joint \textbf{forward} distribution of $(x_0,\epsilon,x_t,u_t)$ obtained by drawing $x_0\sim P_b$ and $\epsilon\sim\mathcal N(0,I)$ independently and forming $x_t$ and $u_t=x_0 - \epsilon$ by forward noising. With $s_t=(x_t,t,c)$, the \emph{ideal forward velocity} induced by $P_b$ is $v_{\mathrm{fwd}}^b(s_t):=\E_{Q_b}[u_t\mid s_t]$. We write $v_{\mathrm{fwd}}^{\mathrm{old}}$ for $b=\mathrm{old}$, and $v_{\mathrm{rvs}}^{\mathrm{old}}:=v_{\mathrm{old}}$ for the field that the reverse sampler actually uses to generate the endpoint distribution.

\begin{lemma}[Ideal NFT guidance is conditional reward covariance]
\label{lem:nft_covariance}
At a fixed state $s_t$, let $Q_{\mathrm{old}}$ denote $Q_b$ with $b=\mathrm{old}$, and let
$q:=\E_{Q_{\mathrm{old}}}[r(x_0,c)\mid s_t]$ and
$g:=\operatorname{Cov}_{Q_{\mathrm{old}}}
[r(x_0,c),\allowbreak u_t\mid s_t]$. If the tilted distributions and moments exist
and $0<q<1$, then
\begin{align}
q\bigl(v_{\mathrm{fwd}}^+-v_{\mathrm{fwd}}^{\mathrm{old}}\bigr)
&=(1-q)\bigl(v_{\mathrm{fwd}}^{\mathrm{old}}
-v_{\mathrm{fwd}}^-\bigr)=g,
\label{eq:nft_conditional_covariance}
\end{align}
where all velocities are evaluated at $s_t$. Hence the positive and negative
branches point along the same reward-tilt direction $g$.
\end{lemma}

Appendix~\ref{app:nft_covariance_proof} provides the proof. The ideal guidance of the two branches thus differs only in step size: under \emph{forward--reverse compatibility}, i.e., if the reverse sampler satisfied $v_{\mathrm{rvs}}^{\mathrm{old}}=v_{\mathrm{fwd}}^{\mathrm{old}}$, either branch alone would suffice. In practice, however, the reverse sampling field generally differs from the ideal forward velocity induced by its own endpoints. We denote this forward--reverse gap by $e:=v_{\mathrm{fwd}}^{\mathrm{old}}(s_t)-v_{\mathrm{rvs}}^{\mathrm{old}}(s_t)$. Proposition~\ref{prop:nft_contamination} shows how it enters the NFT update.

\begin{proposition}[Practical forward-process RL contains a forward--reverse consistency force]
\label{prop:nft_contamination}
Let $\mathcal L_{\mathrm{NFT}}(s_t)$ denote
Equation~\ref{eq:nft_objective} conditioned on $s_t$, $q$ and $g$ follow the definitions of Lemma \ref{lem:nft_covariance}. At the pure on-policy
update point $v_\theta=v_{\mathrm{rvs}}^{\mathrm{old}}$ and for $\beta>0$,
its normalized output-space force is
\begin{equation}
\mathcal F_{\mathrm{NFT}}(s_t)
:=-\frac{1}{2\beta}
\left.\nabla_{v_\theta(s_t)}\mathcal L_{\mathrm{NFT}}(s_t)
\right|_{v_\theta=v_{\mathrm{rvs}}^{\mathrm{old}}}
=2g+(2q-1)e.
\label{eq:nft_force_decomposition}
\end{equation}
\end{proposition}
\noindent\textbf{Why both branches are needed.} Proposition~\ref{prop:nft_contamination} decomposes the NFT force into reward-improving guidance and a self-consistency force weighted by the statewise reward imbalance $2q-1$. The individual branch forces are $\mathcal F_+=g+qe$ and $\mathcal F_-=g-(1-q)e$, which sum to $\mathcal F_{\mathrm{NFT}}$ (proof in Appendix~\ref{app:nft_decomposition_proofs}). Each branch alone thus carries an uncancelled consistency force: the positive branch pulls the velocity toward $v_{\mathrm{fwd}}^{\mathrm{old}}$ with weight $q$, whereas the negative branch pushes it away with weight $1-q$. Joint training therefore does more than double the guidance: the branches reinforce each other along $g$ but oppose each other along $e$, leaving $(2q-1)e$. Without clipping, prompt-level reward centering makes $q$ average to $1/2$ across states, allowing partial cancellation after gradient aggregation (Figure~\ref{fig:nft_branch_balance}(b)).

\noindent\textbf{Why joint training is not enough.} For nonzero consistency gaps, cancellation is exact only if $q=1/2$ at every state, which prompt-level normalization does not guarantee: $\sum_i(2r_i-1)\approx0$ neither balances each state nor implies $\sum_i(2r_i-1)e_i\approx0$, where $e_i$ denotes the one-sample Monte Carlo estimate of $e$. States with $q>1/2$ are systematically pulled toward $v_{\mathrm{fwd}}^{\mathrm{old}}$, while those with $q<1/2$ are pushed away. At $t=1$, the state contains no endpoint information, so without clipping $q=1/2$ and the residual vanishes. As noise decreases, states reveal more endpoint quality, making $q$ more extreme and increasing the weight $\lvert2q-1\rvert$ on the gap. The residual dilutes reward-improving guidance in the aggregate gradient and requires an additional stabilizer (Section~\ref{sec:ema_restoration}).

\begin{insight}[Joint training cancels the consistency force only partially]
\label{ins:nft_self_consistency}
In NFT-style forward-process training, either branch alone carries an uncancelled forward--reverse consistency force that can destabilize training; joint training only partially cancels it across samples, not at each state. A better design should remove this force statewise rather than rely on cross-sample cancellation (Section~\ref{sec:art_guidance}).
\end{insight}
\FloatBarrier

\subsection{Off-Policy EMA: Stability through Restoration, at the Cost of Task-Specific Tuning}
\label{sec:ema_restoration}
\begin{figure*}[!htbp]
\centering
\includegraphics[width=\textwidth]{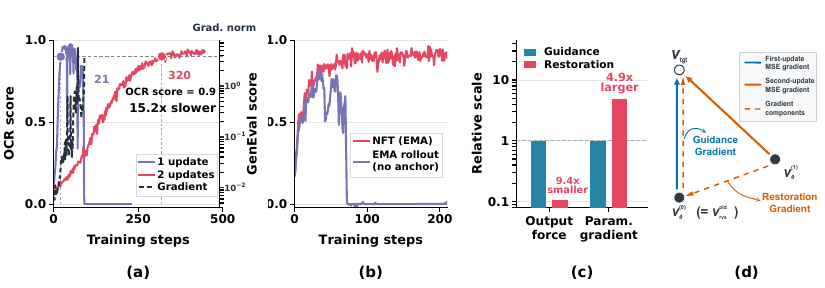}
\vspace{0pt}
\caption{
(a) On-policy DiffusionNFT without EMA ($\beta=1$): increasing number of micro-batches per rollout from 1 to 2 results in  $15.2\times$ slower training. The dashed curve is the one-update run's gradient norm, spiking as the reward collapses. (b) EMA-policy rollouts without the old-policy anchor still collapse. (c) At the second update, averaged over rollouts 0--17, restoration has a $9.4\times$ smaller output-force RMS but a $4.9\times$ larger parameter-gradient norm than guidance. (d) The second MSE update decomposes into guidance and restoration.}
\label{fig:restoration_dynamics}
\vspace{0pt}
\end{figure*}

Even with joint training, pure on-policy NFT collapses after a sharp rise in its pre-clipping parameter-gradient norm (Figure~\ref{fig:restoration_dynamics}(a)). EMA may stabilize training through (1) a slower-moving rollout distribution or (2) the lagged reference $v_{\mathrm{old}}$ retained in the loss (in Eq. \ref{eq:nft_objective}, the $v_{\mathrm{tgt}}$ contains $v_{\mathrm{old}}$, rather than $v_{\theta}$). We show that the latter, namely restoration toward this reference, is the key stabilizing mechanism. Let $d=v_\theta-v_{\mathrm{old}}$ denote the current policy's drift from the rollout old policy and $h=u_t-v_{\mathrm{old}}$ the sample training residual.

\begin{proposition}[EMA brings an explicit restoring term to DiffusionNFT]
\label{prop:nft_guidance_restore}
\begin{equation}
\mathcal L_{\mathrm{NFT}}
=\E_{\mathcal D_{\mathrm{old}}}\!\left[-2\beta(2r-1)
\langle d,h\rangle
+\beta^2\lVert d\rVert_2^2\right]+\mathrm{const},
\label{eq:nft_guidance_restore_decomposition}
\end{equation}
\end{proposition}

where $\mathrm{const}$ is independent of the trainable parameters.

\noindent\textbf{Why EMA stabilizes.} Proposition~\ref{prop:nft_guidance_restore} separates the NFT loss into a linear term carrying both reward guidance and the consistency force, and a quadratic term restoring $v_\theta$ toward $v_{\mathrm{old}}$ (proof in Appendix~\ref{app:nft_decomposition_proofs}). At $d=0$, the restoration gradient vanishes. Pure on-policy NFT therefore receives no restoration, leaving the residual consistency force unopposed. EMA maintains a nonzero drift from the lagged reference, allowing restoration to counteract this residual.

We test these explanations with two ablations. First, retaining EMA-policy rollouts but removing the old-policy anchor from the gradient (i.e. replacing $v_{\mathrm{old}}$ by $v_{\theta}$ in $v_{\mathrm{tgt}}$ of Eq \ref{eq:nft_objective}) still causes collapse (Figure~\ref{fig:restoration_dynamics}(b)): off-policy data alone does not stabilize NFT. Second, without EMA, we split each rollout batch into two disjoint minibatches and update once on each, using the original objective with $\beta=1$. The first update starts at $v_\theta^{(0)}=v_{\mathrm{old}}$, where restoration is zero. After the model moves to $v_\theta^{(1)}$, the second update combines guidance with restoration toward $v_\theta^{(0)}$ (Figure~\ref{fig:restoration_dynamics}(d)). We surprisingly found that this two-update run avoids collapse over the observed training window without EMA, but significantly slows training by 15.2x (Figure~\ref{fig:restoration_dynamics}(a)), supporting restoration as the key stabilizing mechanism at the cost of efficiency in NFT-style forward-process training.

\noindent\textbf{Why EMA is not enough.} EMA controls restoration indirectly through the reference policy's lag: too little restoration risks collapse, while too much slows reward improvement. Its schedule therefore requires task-specific tuning. DiffusionNFT's results illustrate this sensitivity~\citep{zheng2025diffusionnft}: fully on-policy updates collapse, nearly offline updates converge impractically slowly, and the default schedule $\eta=\min(0.001i,0.5)$, where $i$ is the iteration, is unstable on OCR, which requires a maximum retention of $0.999$. In our FLUX experiments, transferring NFT's GenEval-tuned EMA and $\beta$ settings also leads to unstable training and can cause collapse (Figure~\ref{fig:t2i_pair_assignment_main}(a)). Such tuning requirement limits scalability.

Two NFT updates per rollout provide restoration without EMA (Figure~\ref{fig:restoration_dynamics}(d)), but significantly slow training (Figure~\ref{fig:restoration_dynamics}(a)). Counterintuitively, averaged over rollouts 0--17, restoration has a $9.4\times$ smaller output-force RMS than guidance, yet its aggregated parameter-gradient norm is $4.9\times$ larger (Figure~\ref{fig:restoration_dynamics}(c)). This suggests that heterogeneous per-sample guidance largely cancels across states, while restoration survives aggregation and dominates the update, slowing reward optimization. We leave a detailed analysis to our future work.

\begin{insight}[Restoration is crucial for stabilizing NFT-style forward-process training]
\label{ins:statewise_restore}
Restoration stabilizes NFT-style training, but EMA requires task-specific tuning, while two updates per rollout avoid EMA tuning at the cost of slower reward improvement. A better design should adapt restoration to each state's drift without task-specific tuning or overwhelming reward guidance (Section~\ref{sec:art_restoration}).
\end{insight}

\begin{figure}[t]
\centering
\includegraphics[width=\linewidth]{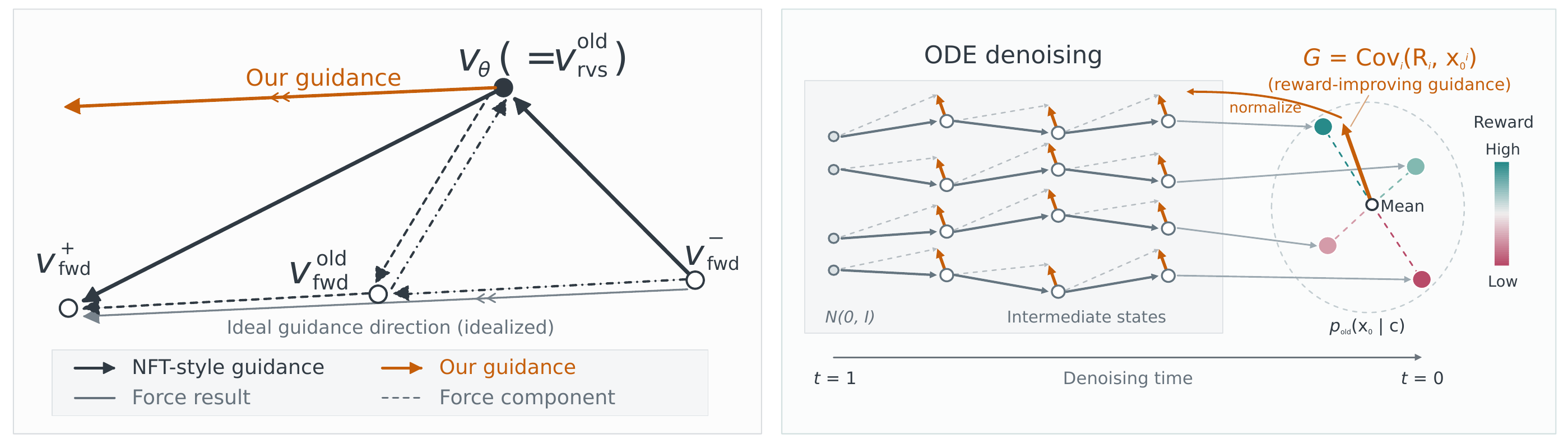}
\vspace{-0.5cm}
\caption{\textbf{Aggregation-first guidance.} \textbf{Left:} NFT's positive and negative branches introduce opposing consistency forces that may not fully cancel (Section~\ref{sec:nft_failure_modes}). DiffusionART avoids these forces by applying $G(c)$ directly at the reverse velocity. \textbf{Right:} We aggregate each prompt's endpoints and rewards into normalized guidance $G(c)$ (Equation~\ref{eq:art_empirical_cov}) and apply it to states visited along this rollout.}
\label{fig:aggregation_first}
\vspace{-0.2cm}
\end{figure}
\section{DiffusionART: Aggregation-first Reward Tilting with Adaptive Restoration}
\label{sec:method}
DiffusionART (Algorithm~\ref{alg:flowart}) combines the advantages of forward- and reverse-process RL, while avoiding their limitations. It applies aggregation-first reward guidance to reverse-rollout states (section \ref{sec:art_guidance}) and uses a second MSE update for drift-adaptive restoration (section \ref{sec:art_restoration}), addressing the two failure modes identified in Section 3. A trust-region mechanism for an ODE solver further avoids excessive policy updates (section \ref{sec:art_mask}).

\subsection{Aggregation-First Guidance}
\label{sec:art_guidance}
The consistency force enters NFT because its regression target is built from the forward velocity $v_{\mathrm{fwd}}^{\mathrm{old}}$, which differs from the reverse field by $e$ (Proposition~\ref{prop:nft_contamination}). We instead aggregate each prompt's endpoints and rewards alone into guidance for reverse-rollout states. At each iteration, we exponentially tilt the current rollout distribution, rather than a fixed pretrained distribution~\citep{potaptchik2025tiltmatching,zhu2026crd}: $p_\tau(x_0\mid c)\propto p_{\mathrm{old}}(x_0\mid c)\exp\bigl(\tau R(x_0,c)\bigr)$. Here $\tau\geq0$ controls the preference for higher-reward endpoints, and $\tau=0$ recovers $p_{\mathrm{old}}$. The first-order change in the endpoint mean is a covariance:

\begin{lemma}[Reward tilting moves the endpoint mean along the reward--endpoint covariance]
\label{lem:art_tilt}
Under standard integrability conditions, the derivative of the tilted endpoint mean at $\tau = 0$ equals the reward–endpoint covariance:
$\left.\frac{\partial}{\partial\tau}\E_{p_\tau}[x_0\mid c]\right|_{\tau=0}
=\operatorname{Cov}_{p_{\mathrm{old}}}(R,x_0\mid c)$.
\end{lemma}
Lemma~\ref{lem:art_tilt}, proved in Appendix~\ref{app:art_details}, shows that up-weighting higher-reward endpoints shifts their mean along $\operatorname{Cov}_{p_{\mathrm{old}}}(R,x_0\mid c)$. Under rectified flow, $u_t=x_0-\epsilon$ with $\epsilon$ independent of $x_0$, so $\operatorname{Cov}_{Q_{\mathrm{old}}}(R,u_t\mid c)=\operatorname{Cov}_{p_{\mathrm{old}}}(R,x_0\mid c)$ at every noise level. This direction is the prompt-level counterpart of ideal NFT guidance $g$ (Lemma~\ref{lem:nft_covariance}): without clipping, $r$ is affine in $R$, and at $t=1$, when the forward state reveals no endpoint information, $g$ equals this covariance up to a positive factor. We estimate the covariance from each prompt's $K$ rollout endpoints and rewards, then normalize it to obtain $G(c)$:\par

\begingroup
\small

\begin{equation}
\boxed{\overline{\operatorname{Cov}}(R,x_0\mid c)
=\frac1K\sum_{i=1}^{K}(R_i-\bar R_c)(x_0^{(i)}-\bar x_{0,c}),
\qquad G(c)=\frac{\overline{\operatorname{Cov}}}{\rms(\overline{\operatorname{Cov}})}}.
\label{eq:art_empirical_cov}
\end{equation}
\endgroup
Here $\bar R_c$ and $\bar x_{0,c}$ are the group means.
RMS normalization fixes the guidance magnitude across prompts and reward models, removing differences in reward scale and spread without a guidance-strength coefficient. We apply $G(c)$ to every rollout state $s_t$ of prompt $c$ and freeze the target $v_{\mathrm{tgt}}^{\mathrm{ART}}(s_t)=v_\theta^{(0)}(s_t)+G(c)$, where $v_\theta^{(0)}=v_{\mathrm{old}}$ is the prediction cached during rollout. At $v_\theta=v_\theta^{(0)}$, every state receives an output-space force proportional to $G(c)$, with no NFT consistency term (Equations~\ref{eq:art_update_components} and~\ref{eq:nft_force_decomposition}; Insight~\ref{ins:nft_self_consistency}). Figure~\ref{fig:aggregation_first} summarizes the construction.

\vspace{0pt}
\subsection{Two-Update Restoration: Drift-Adaptive Stability without EMA Tuning}

\label{sec:art_restoration}
Instead of tuning an EMA restoration schedule (Section~\ref{sec:ema_restoration}), we split each rollout batch into two disjoint minibatches and apply one MSE update to each. Let $v_\theta^{(0)}$ and $v_\theta^{(1)}$ denote the velocity fields before and after the first update. Both updates use the frozen target $v_{\mathrm{tgt}}^{\mathrm{ART}}=v_\theta^{(0)}+G(c)$. For each state $s_t$, the loss is $\mathcal L_\theta(s_t)=w(s_t)\operatorname{MSE}(v_\theta(s_t),v_{\mathrm{tgt}}^{\mathrm{ART}}(s_t))$, where $\operatorname{MSE}(a,b)=D^{-1}\lVert a-b\rVert_2^2$, $D$ is the number of output coordinates, and $w$ combines timestep weighting and masking. Let $d(s_t)=v_\theta(s_t)-v_\theta^{(0)}(s_t)$ denote the current displacement. The loss and output gradient before each update are:\par

\begingroup
\small
\begin{equation}
\setlength{\dashlinedash}{2pt}
\setlength{\dashlinegap}{2pt}
\setlength{\arraycolsep}{8pt}\begin{array}{@{}w{c}{0.43\linewidth}:w{c}{0.43\linewidth}@{}}
\mathcal L^{(0)}=w\operatorname{MSE}(v_\theta^{(0)},v_{\mathrm{tgt}}^{\mathrm{ART}})
=\dfrac{w}{D}\lVert G(c)\rVert_2^2
& \mathcal L^{(1)}=w\operatorname{MSE}(v_\theta^{(1)},v_{\mathrm{tgt}}^{\mathrm{ART}}) 
=\dfrac{w}{D}\lVert G(c)-d(s_t)\rVert_2^2 \\[6pt]
-\left.\nabla_{v_\theta}\mathcal L_\theta\right|_{v_\theta=v_\theta^{(0)}}
=\dfrac{2w}{D}\,G(c)
& -\left.\nabla_{v_\theta}\mathcal L_\theta\right|_{v_\theta=v_\theta^{(1)}}
=\dfrac{2w}{D}\bigl(G(c)-d(s_t)\bigr).
\end{array}
\label{eq:art_update_components}
\end{equation}
\endgroup
\par\noindent The second update combines the original guidance $G(c)$ with restoration $-d(s_t)$ toward $v_\theta^{(0)}$. Restoration grows with each state's drift, counteracting large displacements while leaving guidance dominant when drift is small, without additional tuning. In NFT, cross-state cancellation weakens guidance, allowing restoration to dominate and slow training (Section~\ref{sec:ema_restoration}). DiffusionART instead shares $G(c)$ across each prompt's states, reducing this cancellation and improving stability without sacrificing efficiency (Insight~\ref{ins:statewise_restore}; Figure~\ref{fig:t2i_ablation_curves}).

\subsection{Trust-Region Masking for ODE Rollouts}
\label{sec:art_mask}
Despite restoration, applying guidance gradient in both updates can produce an excessive policy change within a single rollout batch, potentially destabilizing training. Inspired by PPO clipping~\citep{schulman2017proximal}, we constrain it with a trust-region mask as a stabilizing trick. However, ODE transitions are deterministic, providing no usable finite KL to measure policy change. We address this by following Flow-GRPO's ODE-to-SDE formulation~\citep{liu2025flowgrpo} to define reference Gaussian transitions for the current and old policies: $q_\theta(\cdot\mid s_t)=\mathcal N(\mu(s_t,v_\theta(s_t)),\sigma_{\rm ref}^2(t)I)$ and $q_{\rm old}(\cdot\mid s_t)=\mathcal N(\mu(s_t,v_\theta^{(0)}(s_t)),\sigma_{\rm ref}^2(t)I)$. Based on this, we propose a KL proxy to measure the drift:
\begin{equation}
 m(s_t)=\frac{1}{D}D_{\rm KL}\!\left(q_\theta(\cdot\mid s_t)\,\|\,q_{\rm old}(\cdot\mid s_t)\right)
 =\frac{\|\mu(s_t,v_\theta(s_t))-\mu(s_t,v_\theta^{(0)}(s_t))\|_2^2}{2D\sigma_{\rm ref}^2(t)}.
 \label{eq:art_kl_proxy}
\end{equation}
The proxy uses velocity predictions and rollouts remain ODE trajectories. Here $\mu(s_t,v)$ is the mean of the next latent after one reference SDE step with velocity $v$, and $\sigma_{\rm ref}(t)$ is the policy-independent noise scale of that step (formulas in Appendix~\ref{app:art_mask}).

Inspired by Flow-DPPO~\citep{ping2026flowdppo}, we mask excessive outward drift while retaining inward corrections. We identify outward updates by a positive inner product between the displacement $d(s_t)$ and the second update's output-space descent direction $G(c)-d(s_t)$:
\begin{equation}
 w(s_t)=t\left[1-\mathbf1\{m(s_t)\geq\delta\}\,
 \mathbf1\{\langle d(s_t),G(c)-d(s_t)\rangle>0\}\right].
 \label{eq:art_trust_mask}
\end{equation}
The two conditions detect excessive drift and outward updates, respectively. We fix the drift threshold $\delta$ to $10^{-5}$ throughout and use the factor $t$ to emphasize high-noise states.

\begingroup

\begin{algorithm}[!htbp]
\caption{DiffusionART: one training iteration}\label{alg:flowart}
\small
\begin{algorithmic}[1]
\State Sample with an ODE solver using $v_\theta^{(0)}=v_\theta$. \hfill\textcolor{black!50}{\textbf{Rollout Stage}}
\Statex \hspace{\algorithmicindent}For each prompt $c$, collect $K$ endpoints with rewards $R$, states $s_t$, and predictions $v_\theta^{(0)}(s_t)$.
\State Compute $G(c)$ by Equation~\ref{eq:art_empirical_cov} for each prompt $c$.
\State Freeze $v_{\mathrm{tgt}}^{\mathrm{ART}}(s_t)=v_\theta^{(0)}(s_t)+G(c)$ for each state $s_t$ of prompt $c$.

\Statex \leavevmode\leaders\hbox{\rule[0.5ex]{2pt}{0.3pt}\hskip2pt}\hfill\mbox{}
\State Split the samples into two equal-sized disjoint minibatches $B_1,B_2$. \hfill\textcolor{black!50}{\textbf{Training Stage}}
\State For each $B\in\{B_1,B_2\}$, recompute $w(s_t)$ and update $\theta$ with one step on

\Statex \hspace{\algorithmicindent}$\mathcal L(B;\theta)=\dfrac{1}{|B|}\sum_{s_t\in B}w(s_t)\operatorname{MSE}(v_\theta(s_t),v_{\mathrm{tgt}}^{\mathrm{ART}}(s_t))$.
\end{algorithmic}
\end{algorithm}
\endgroup

\subsection{Extension to Pairwise Reward Models}
\label{sec:art_pairwise}
Unlike math or code, visual outputs often lack objectively verifiable correctness criteria, making pairwise preferences a natural alternative to scalar scores for reward modeling \citep{kirstain2023pickapic}. \textbf{DiffusionART-Pairwise} is a variant of DiffusionART which naturally supports online pairwise preference feedback~\citep{kirstain2023pickapic}. For each prompt, we divide $K$
  endpoints into $K/2$ disjoint pairs and assign winner/loser labels $(1,0)$, or $(1/2,1/2)$ for ties. These labels replace scalar
  rewards in the otherwise unchanged covariance and training pipeline. Appendix~\ref{app:art_pairwise} provides the covariance
  formulation and pseudocode.

\begingroup

\section{Experiments}
\label{sec:experiments}

We study the efficiency, stability, and scalability of DiffusionART, progressing from text-to-image (T2I) generation to larger models and higher-dimensional text-to-video (T2V) and image-to-video (I2V) generation.

\subsection{Text-to-Image Generation}
\label{sec:t2i}

\noindent\textbf{Setup.}\label{sec:t2i_main}
We compare DiffusionART with Flow-GRPO, Flow-GRPO-Fast (CPS)~\citep{liu2025flowgrpo,wang2025coefficientspreservingsamplingreinforcementlearning}, DanceGRPO~\citep{xue2025dancegrpo}, Flow-DPPO~\citep{ping2026flowdppo}, and DiffusionNFT~\citep{zheng2025diffusionnft} with and without EMA, denoted DiffusionNFT (EMA) and DiffusionNFT (on-policy). All methods train LoRA adapters~\citep{hu2021lora} with rank $r=32$ and scaling factor $\alpha=64$ at $512\times512$ resolution, without CFG or reference-policy KL regularization. On SD3.5-M, we use a learning rate of $3\times10^{-4}$ and 10-step rollouts. DiffusionNFT uses the original per-task hyperparameters~\citep{zheng2025diffusionnft} to maximize its baseline performance.

\noindent\textbf{Single- and multi-reward optimization.}
We evaluate SD3.5-M~\citep{stabilityai2024sd35} on GenEval~\citep{ghosh2023geneval}, OCR, and PickScore~\citep{kirstain2023pickapic}, and further jointly optimize GenEval, PickScore, CLIPScore~\citep{hessel2021clipscore}, and HPSv2.1~\citep{wu2023hpsv2}. Figure~\ref{fig:t2i_main_curves}(a--d) shows that DiffusionART behaves similarly under single- and multi-reward training. Figure~\ref{fig:t2i_main_curves}(d) plots the total reward, with components in Appendix~\ref{app:t2i_multi}. Without per-task EMA tuning, DiffusionART avoids the severe collapse of DiffusionNFT (on-policy) and achieves training efficiency comparable to DiffusionNFT (EMA). It is also much faster than Flow-GRPO, reaching Flow-GRPO's step-500 reward with $10.0\times$, $5.2\times$, and $9.6\times$ fewer training steps on GenEval, OCR, and PickScore, respectively.

\begin{figure}[!htbp]
\centering
\setlength{\abovecaptionskip}{2pt}
\setlength{\belowcaptionskip}{0pt}
\includegraphics[width=\linewidth]{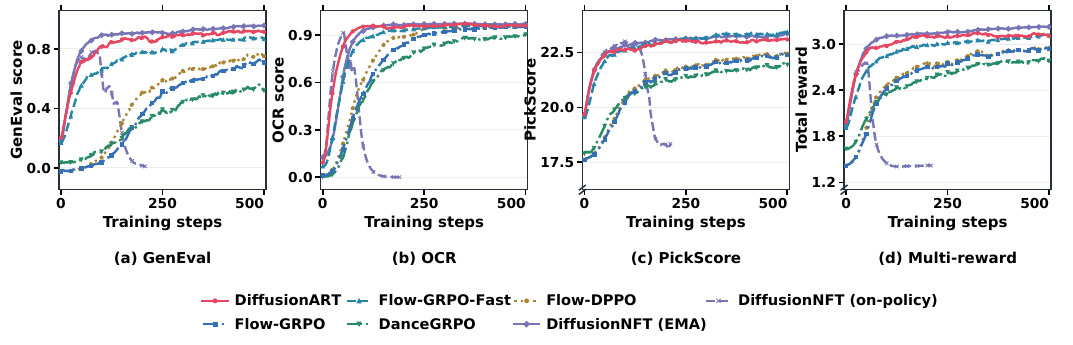}
\caption{Training rewards: (a--c) single- and (d) multi-reward SD3.5-M.}
\label{fig:t2i_main_curves}
\end{figure}

\noindent\textbf{Scaling to larger models.}
We extend the same recipe to FLUX.1-dev (12B)~\citep{blackforest2024flux} on GenEval, using six-step rollouts and its built-in guidance embedding at scale 3.5, without additional CFG. For DiffusionNFT, we retain $\beta=1$ and rollout-policy EMA decay $\min(0.001n,0.5)$ from the SD3.5-M GenEval baseline, where $n$ counts successful training updates. As shown in Figure~\ref{fig:t2i_pair_assignment_main}(a), DiffusionART learns fastest among all methods, while DiffusionNFT exhibits pronounced training-reward fluctuations because these hyperparameters are not retuned for FLUX. Table~\ref{tab:t2i_flux} in Appendix~\ref{app:t2i_flux} reports the GenEval evaluation scores of different methods at step 500.

\begin{figure}[H]
  \centering
  \includegraphics[width=\linewidth]{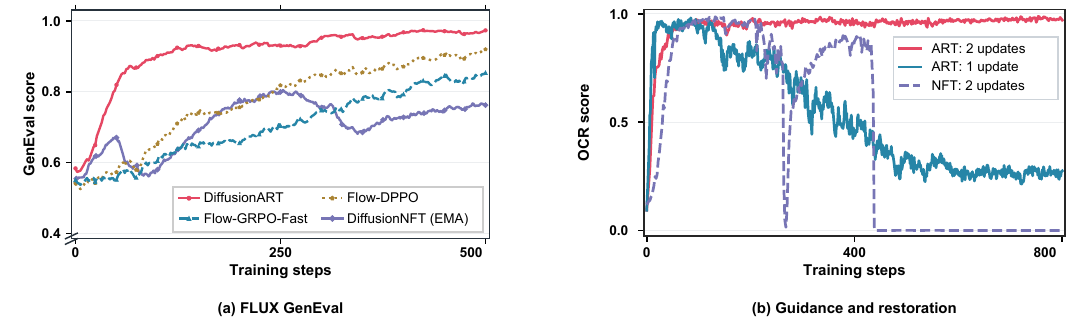}
  \vspace{-2pt}
  \caption{(a) FLUX.1-dev on GenEval. (b) Guidance and restoration ablation on OCR.}
  \label{fig:t2i_pair_assignment_main}\label{fig:t2i_ablation_curves}\vspace{-4pt}
\end{figure}

\noindent\textbf{Ablation of guidance and restoration.}\label{sec:t2i_restoration}
Figure~\ref{fig:t2i_pair_assignment_main}(b) compares \textit{ART: 1 update} and \textit{ART: 2 updates}, two DiffusionART variants with the same covariance guidance, against \textit{NFT: 2 updates} with a carefully tuned $\beta=0.1$. As explained in Appendix~\ref{app:art_guidance_approximation}, \textit{ART: 1 update} learns quickly but gradually deteriorates. The second update in \textit{ART: 2 updates} introduces restoration and provides long-term stability. In contrast, \textit{NFT: 2 updates} briefly collapses to near zero before recovering. These results highlight the importance of both guidance type and restoration.

\subsection{Text-to-Video Generation}
\label{sec:t2v}

\noindent\textbf{Setup.}
We fine-tune all parameters of the CFG-distilled HunyuanVideo-13B model~\citep{kong2024hunyuanvideo} with a learning rate of $10^{-5}$ and no KL regularization. Each group-based rollout uses 32 prompts and 8 videos per prompt. We generate 53-frame videos at $480\times480$ and 8 FPS with 16 sampling steps. VideoReward~\citep{liu2025videoreward} scores videos at 4 FPS for training and evaluation. All methods use matched TA:VQ:MQ weights of $7{:}2{:}1$ to improve output quality while preserving prompt following. Figure~\ref{fig:t2v_overall_curves}(a--c) compares all three rewards over 400 training steps; Appendix~\ref{app:t2v_matched_rewards} reports an additional VQ:MQ $=1{:}1$ comparison.

\noindent\textbf{Performance.}
As shown in Figure~\ref{fig:t2v_overall_curves}, DiffusionART surpasses all baselines in both VQ and MQ while keeping TA comparable with the strongest baseline. DiffusionART reaches training VQ of 6 and MQ of 3 after 29 and 46 steps, respectively; Flow-GRPO-Fast, DanceGRPO, and DiffusionNFT (EMA) do not reach these levels within their recorded trajectories. For VQ of 5.5, DiffusionART needs 24 steps and 63.4 GPU hours, versus 273 steps and 1,401.4 GPU hours for DiffusionNFT (EMA): \textbf{11.4$\times$ fewer training steps and 22.1$\times$ fewer GPU hours}. It retains its gains through 400 steps and achieves evaluation VQ and MQ of 8.970 and 4.224, as reported in Table~\ref{tab:t2v_overall}. DiffusionART improves evaluation VQ and MQ by 31.4\% and 57.4\%, respectively, over the strongest baseline.

\begin{figure}[H]
\centering
\includegraphics[width=\linewidth]{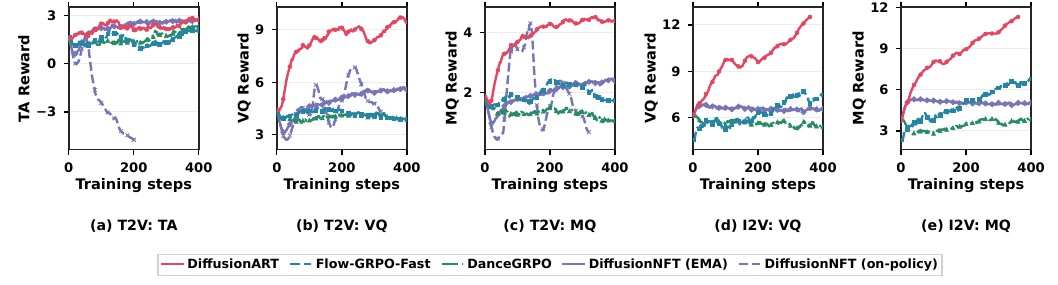}
\setlength{\abovecaptionskip}{2pt}
\setlength{\belowcaptionskip}{0pt}
\caption{\textbf{Efficient reward improvement across video generation.} (a--c) T2V: DiffusionART improves VQ and MQ fastest while maintaining TA comparable to the strongest baseline. (d--e) I2V: DiffusionART rapidly improves both rewards, whereas the baselines improve slowly or plateau.}
\label{fig:t2v_overall_curves}
\label{fig:i2v_overall_curves}
\end{figure}

\noindent\textbf{Ablation of trust-region masking.}\label{app:mask_ablation}
We remove only the trust-region mask described in Section~\ref{sec:art_mask}, retaining normalized covariance guidance, timestep weighting, and two disjoint MSE updates. With VQ:MQ $=1{:}1$ weights, Figure~\ref{fig:mask_ablation} shows that the no-mask run drops around step 200 before recovering, whereas the masked run improves steadily. Masking helps limit excessive updates, serving as an additional stability safeguard; Appendix~\ref{app:mask_diagnostics} provides further diagnostics of gradients and velocity drift, and Appendix~\ref{app:mask_image} reports the trust-region ablation on T2I tasks.

\begin{figure}[!t]
\centering
\begin{minipage}[c]{0.60\linewidth}
\centering
\includegraphics[width=\linewidth]{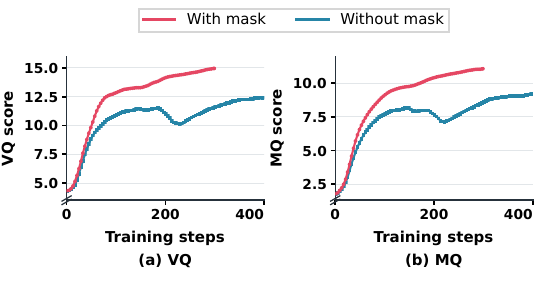}
\setlength{\abovecaptionskip}{2pt}
\caption{T2V with and without trust-region masking.}
\label{fig:mask_ablation}
\end{minipage}\hfill
\begin{minipage}[c]{0.37\linewidth}
\centering\footnotesize
\setlength{\abovecaptionskip}{2pt}
\setlength{\belowcaptionskip}{4pt}
\captionof{table}{VQ-only evaluation and GPU hours at step 400.}
\label{tab:t2v_pairwise}
\setlength{\tabcolsep}{4pt}
\begin{tabular}{lrr}
\toprule
Method & GPU Hours $\downarrow$ & VQ $\uparrow$ \\
\midrule
Scalar, $K=8$ & 891.1 & \textbf{13.637} \\
\midrule
\rowcolor{blue!8}
Pairwise, $K=2$ & \textbf{260.8} & 12.235 \\
Pairwise, $K=4$ & 465.6 & 10.883 \\
Pairwise, $K=8$ & 883.3 & 12.092 \\
\bottomrule
\end{tabular}
\end{minipage}
\end{figure}

\noindent\textbf{Pairwise feedback and ablation of group size.}
The performance of DiffusionART with pairwise train reward is close to that obtained with scalar feedback, using scalar reward as the ground truth evaluation metric. Meanwhile, DiffusionART is insensitive to the group size $K$, and a smaller $K$ substantially reduces GPU hours. Table~\ref{tab:t2v_pairwise} shows that, under the same 400-step budget with pairwise feedback, $K=2$ reaches a final VQ of 12.235, no lower than 12.092 for $K=8$. Because $K=2$ uses only a quarter of the rollout videos and training samples, it consumes only 260.8 GPU hours over the 400 steps, 3.4$\times$ fewer than the 883.3 GPU hours of $K=8$. One pair per prompt thus suffices for stable and compute-efficient optimization in this setting.

\par\addvspace{4pt}

\FloatBarrier
\subsection{Image-to-Video Generation}
\label{sec:i2v}

\noindent\textbf{Setup.}
We fine-tune all parameters of the CFG- and step-distilled Wan2.1-I2V-14B model~\citep{wan2025video} released by LightX2V~\citep{lightx2v2025}. Each rollout uses 32 prompts, FLUX.1-dev reference images, and 8 videos per prompt. We generate 53-frame videos at $400\times640$ and 8 FPS with four sampling steps. All methods use VQ:MQ $=1{:}1$ weights, a learning rate of $10^{-5}$, and no KL regularization. Reward-model settings follow Section~\ref{sec:t2v}.

\noindent\textbf{Performance.}
DiffusionART improves both VQ and MQ efficiently, whereas the baselines improve slowly or plateau (Figure~\ref{fig:i2v_overall_curves}(d--e)). It reaches MQ of 6 in 45.8 GPU hours versus 220.3 for Flow-GRPO-Fast, a \textbf{4.8$\times$ reduction}. DiffusionNFT plateaus early and remains near that plateau through 400 updates, which we attribute to the residual self-consistency force diluting reward-improving guidance after gradient aggregation. Evaluation results are in Table~\ref{tab:i2v_eval}.

\section{Conclusion}
We studied why NFT-style forward-process diffusion post-training can be unstable and found that its objective can introduce an implicit forward--reverse
self-consistency force that is not reliably canceled across states, while stronger restoration can improve stability at the cost of slower reward optimization. DiffusionART presents a simple but effective way to systematically cancel the unstable force by aggregating endpoint reward guidance before assigning it to reverse-rollout states, and then using a second update with a trust-region mechanism to provide drift-adaptive restoration for stabilization. Experiments on image and video generation show efficient, stable reward improvement, including with pairwise feedback and small rollout groups. Beyond the algorithm, our work provides insights into forward-process training, offering a clearer understanding of diffusion reinforcement learning and a path toward scalable diffusion post-training for visual generation.

\section{Acknowledgement}
We thank Kaiwen Zheng and Huayu Chen for the insightful discussion about the intuition behind DiffusionNFT and the valuable suggestion for the experiment design.

\FloatBarrier
\endgroup
\bibliography{references}

@article{black2023ddpo,
  title={Training Diffusion Models with Reinforcement Learning},
  author={Black, Kevin and Janner, Michael and Du, Yilun and Kostrikov, Ilya and Levine, Sergey},
  journal={arXiv preprint arXiv:2305.13301},
  year={2023}
}

@article{lipman2022flowmatching,
  title={Flow Matching for Generative Modeling},
  author={Lipman, Yaron and Chen, Ricky T. Q. and Ben-Hamu, Heli and Nickel, Maximilian and Le, Matt},
  journal={arXiv preprint arXiv:2210.02747},
  year={2022}
}

@article{liu2025flowgrpo,
  title={{Flow-GRPO}: Training Flow Matching Models via Online {RL}},
  author={Liu, Jie and Liu, Gongye and Liang, Jiajun and Li, Yangguang and Liu, Jiaheng and Wang, Xintao and Wan, Pengfei and Zhang, Di and Ouyang, Wanli},
  journal={arXiv preprint arXiv:2505.05470},
  year={2025}
}

@article{potaptchik2025tiltmatching,
  title={Tilt Matching for Scalable Sampling and Fine-Tuning},
  author={Potaptchik, Peter and Lee, Cheuk-Kit and Albergo, Michael S.},
  journal={arXiv preprint arXiv:2512.21829},
  year={2025}
}

@article{wallace2023diffusiondpo,
  title={Diffusion Model Alignment Using Direct Preference Optimization},
  author={Wallace, Bram and Dang, Meihua and Rafailov, Rafael and Zhou, Linqi and Lou, Aaron and Purushwalkam, Senthil and Ermon, Stefano and Xiong, Caiming and Joty, Shafiq and Naik, Nikhil},
  journal={arXiv preprint arXiv:2311.12908},
  year={2023}
}

@article{zheng2025diffusionnft,
  title={{DiffusionNFT}: Online Diffusion Reinforcement with Forward Process},
  author={Zheng, Kaiwen and Chen, Huayu and Ye, Haotian and Wang, Haoxiang and Zhang, Qinsheng and Jiang, Kai and Su, Hang and Ermon, Stefano and Zhu, Jun and Liu, Ming-Yu},
  journal={arXiv preprint arXiv:2509.16117},
  year={2025}
}

@article{xu2026pathspace,
  title={Designing Reinforcement Learning for Diffusion Models: A Unified Path-Space View},
  author={Xu, Yixian and Zhang, Yuanrui and Luo, Shengjie and Wang, Liwei and He, Di},
  journal={arXiv preprint arXiv:2608.14430},
  year={2026}
}

@article{ping2026flowdppo,
  title={{Flow-DPPO}: Divergence Proximal Policy Optimization for Flow Matching Models},
  author={Ping, Bowen and Zhou, Xiangxin and Qi, Penghui and Luo, Minnan and Bo, Liefeng and Pang, Tianyu},
  journal={arXiv preprint arXiv:2606.11025},
  year={2026}
}

@article{zhu2026crd,
  title={Diffusion Reinforcement Learning via Centered Reward Distillation},
  author={Zhu, Yuanzhi and Wang, Xi and Lathuiliere, Stephane and Kalogeiton, Vicky},
  journal={arXiv preprint arXiv:2603.14128},
  year={2026}
}

@misc{stabilityai2024sd35,
  title={Stable Diffusion 3.5 Medium},
  author={{Stability AI}},
  year={2024},
  howpublished={\url{https://huggingface.co/stabilityai/stable-diffusion-3.5-medium}}
}

@article{ghosh2023geneval,
  title={{GenEval}: An Object-Focused Framework for Evaluating Text-to-Image Alignment},
  author={Ghosh, Dhruba and Hajishirzi, Hanna and Schmidt, Ludwig},
  journal={arXiv preprint arXiv:2310.11513},
  year={2023}
}

@article{kirstain2023pickapic,
  title={Pick-a-Pic: An Open Dataset of User Preferences for Text-to-Image Generation},
  author={Kirstain, Yuval and Polyak, Adam and Singer, Uriel and Matiana, Shahbuland and Penna, Joe and Levy, Omer},
  journal={arXiv preprint arXiv:2305.01569},
  year={2023}
}

@article{xue2025dancegrpo,
  title={{DanceGRPO}: Unleashing {GRPO} on Visual Generation},
  author={Xue, Zeyue and Wu, Jie and Gao, Yu and Kong, Fangyuan and Zhu, Lingting and others},
  journal={arXiv preprint arXiv:2505.07818},
  year={2025}
}

@article{hu2021lora,
  title={{LoRA}: Low-Rank Adaptation of Large Language Models},
  author={Hu, Edward J. and Shen, Yelong and Wallis, Phillip and Allen-Zhu, Zeyuan and Li, Yuanzhi and others},
  journal={arXiv preprint arXiv:2106.09685},
  year={2021}
}

@misc{blackforest2024flux,
  title={{FLUX.1 [dev]}},
  author={{Black Forest Labs}},
  year={2024},
  howpublished={\url{https://huggingface.co/black-forest-labs/FLUX.1-dev}}
}

@article{hessel2021clipscore,
  title={{CLIPScore}: A Reference-free Evaluation Metric for Image Captioning},
  author={Hessel, Jack and Holtzman, Ari and Forbes, Maxwell and Le Bras, Ronan and Choi, Yejin},
  journal={arXiv preprint arXiv:2104.08718},
  year={2021}
}

@article{wu2023hpsv2,
  title={Human Preference Score v2: A Solid Benchmark for Evaluating Human Preferences of Text-to-Image Synthesis},
  author={Wu, Xiaoshi and Hao, Yiming and Sun, Keqiang and Chen, Yixiong and Zhu, Feng and Zhao, Rui and Li, Hongsheng},
  journal={arXiv preprint arXiv:2306.09341},
  year={2023}
}

@article{kong2024hunyuanvideo,
  title={{HunyuanVideo}: A Systematic Framework for Large Video Generative Models},
  author={Kong, Weijie and Tian, Qi and Zhang, Zijian and Min, Rox and Dai, Zuozhuo and others},
  journal={arXiv preprint arXiv:2412.03603},
  year={2024}
}

@article{liu2025videoreward,
  title={Improving Video Generation with Human Feedback},
  author={Liu, Jie and Liu, Gongye and Liang, Jiajun and Yuan, Ziyang and Liu, Xiaokun and others},
  journal={arXiv preprint arXiv:2501.13918},
  year={2025}
}

@article{wan2025video,
  title={{Wan}: Open and Advanced Large-Scale Video Generative Models},
  author={{Wan Team}},
  journal={arXiv preprint arXiv:2503.20314},
  year={2025}
}

@misc{lightx2v2025,
  title={{Wan2.1-I2V-14B-720P-StepDistill-CfgDistill-Lightx2v}},
  author={{LightX2V}},
  year={2025},
  howpublished={\url{https://huggingface.co/lightx2v/Wan2.1-I2V-14B-720P-StepDistill-CfgDistill-Lightx2v}}
}

@article{xu2023imagereward,
  title={{ImageReward}: Learning and Evaluating Human Preferences for Text-to-Image Generation},
  author={Xu, Jiazheng and Liu, Xiao and Wu, Yuchen and Tong, Yuxuan and Li, Qinkai and Ding, Ming and Tang, Jie and Dong, Yuxiao},
  journal={arXiv preprint arXiv:2304.05977},
  year={2023}
}

@article{xu2026qwenimage2rl,
  title={{Qwen-Image-2.0-RL} Technical Report},
  author={Xu, Yixian and Gao, Kaiyuan and Chen, Yuxiang and Chen, Yilei and Tang, Zecheng and others},
  journal={arXiv preprint arXiv:2606.27608},
  year={2026}
}

@article{gao2025seedance,
  title={{Seedance 1.0}: Exploring the Boundaries of Video Generation Models},
  author={Gao, Yu and Guo, Haoyuan and Hoang, Tuyen and Huang, Weilin and Jiang, Lu and others},
  journal={arXiv preprint arXiv:2506.09113},
  year={2025}
}

@article{he2025tempflowgrpo,
  title={{TempFlow-GRPO}: When Timing Matters for {GRPO} in Flow Models},
  author={He, Xiaoxuan and Fu, Siming and Zhao, Yuke and Li, Wanli and Yang, Jian and Yin, Dacheng and Rao, Fengyun and Zhang, Bo},
  journal={arXiv preprint arXiv:2508.04324},
  year={2025}
}

@inproceedings{wang2026grpoguard,
  title={{GRPO-Guard}: Mitigating Implicit Over-Optimization in Flow Matching via Regulated Clipping},
  author={Wang, Jing and Liang, Jiajun and Liu, Jie and Liu, Henglin and Liu, Gongye and Zheng, Jun and Pang, Wanyuan and Ma, Ao and Xie, Zhenyu and Wang, Xintao and Wang, Meng and Wan, Pengfei and Liang, Xiaodan},
  booktitle={Proceedings of the IEEE/CVF Conference on Computer Vision and Pattern Recognition},
  pages={5988--5998},
  year={2026}
}

@misc{wang2025coefficientspreservingsamplingreinforcementlearning,
  title={Coefficients-Preserving Sampling for Reinforcement Learning with Flow Matching},
  author={Feng Wang and Zihao Yu},
  year={2025},
  eprint={2509.05952},
  archivePrefix={arXiv},
  primaryClass={cs.CV},
  url={https://arxiv.org/abs/2509.05952}
}

@misc{xue2026advantageweightedmatchingaligning,
  title={Advantage Weighted Matching: Aligning RL with Pretraining in Diffusion Models},
  author={Shuchen Xue and Chongjian Ge and Shilong Zhang and Yichen Li and Zhi-Ming Ma},
  year={2026},
  eprint={2509.25050},
  archivePrefix={arXiv},
  primaryClass={cs.LG},
  url={https://arxiv.org/abs/2509.25050}
}

@misc{liu2022flowstraightfastlearning,
  title={Flow Straight and Fast: Learning to Generate and Transfer Data with Rectified Flow},
  author={Xingchao Liu and Chengyue Gong and Qiang Liu},
  year={2022},
  eprint={2209.03003},
  archivePrefix={arXiv},
  primaryClass={cs.LG},
  url={https://arxiv.org/abs/2209.03003}
}

@article{schulman2017proximal,
  title={Proximal Policy Optimization Algorithms},
  author={Schulman, John and Wolski, Filip and Dhariwal, Prafulla and Radford, Alec and Klimov, Oleg},
  journal={arXiv preprint arXiv:1707.06347},
  year={2017}
}
\bibliographystyle{unsrtnat}
\clearpage
\appendix
\addtocontents{toc}{\protect\setcounter{tocdepth}{2}}
\begingroup
\renewcommand{\contentsname}{Appendix Contents}
\tableofcontents
\endgroup
\clearpage
\begingroup
\setlength{\textfloatsep}{8pt plus 2pt minus 2pt}
\setlength{\intextsep}{8pt plus 2pt minus 2pt}
\setlength{\floatsep}{6pt plus 2pt minus 2pt}
\section{Additional Experimental Results}
\label{app:additional_experiments}\subsection{Scaling to FLUX.1-dev}
\label{app:t2i_flux}
\begin{table}[!htbp]
\centering\footnotesize
\caption{\textbf{FLUX.1-dev GenEval evaluation.}}
\label{tab:t2i_flux}
\setlength{\tabcolsep}{3pt}
\begin{tabular}{lr}
\toprule
Method & Eval reward\\
\midrule
Flow-GRPO-Fast & 0.8374\\
Flow-DPPO & 0.9183\\
DiffusionNFT (EMA) & 0.8490\\
\textbf{\method{}} & \textbf{0.9684}\\
\bottomrule
\end{tabular}

\end{table}

\subsection{T2V Evaluation}
\label{app:t2v_evaluation}
Table~\ref{tab:t2v_overall} reports the latest evaluation scores under matched TA:VQ:MQ $=7{:}2{:}1$ training weights. DiffusionART improves evaluation VQ and MQ by 31.4\% and 57.4\%, respectively, over the strongest baseline.

\begin{table}[H]
\centering\footnotesize
\setlength{\abovecaptionskip}{2pt}
\caption{\textbf{T2V evaluation.} Matched TA:VQ:MQ $=7{:}2{:}1$.}
\label{tab:t2v_overall}
\setlength{\tabcolsep}{4pt}
\begin{tabular}{lrrr}
\toprule
Method & VQ $\uparrow$ & MQ $\uparrow$ & TA $\uparrow$ \\
\midrule
Pretrained & 5.195 & 1.943 & 1.118 \\
Flow-GRPO-Fast & 3.149 & 1.874 & 1.667 \\
DanceGRPO & 3.922 & 1.206 & 1.685 \\
DiffusionNFT (EMA) & 6.827 & 2.683 & 2.096 \\
DiffusionNFT (on-policy) & 5.543 & 1.839 & -4.480 \\
\midrule
\rowcolor{blue!8}
DiffusionART & \textbf{8.970} & \textbf{4.224} & \textbf{2.111} \\
\bottomrule
\end{tabular}
\end{table}

\subsection{I2V Evaluation}
\label{app:i2v_evaluation}
\begin{table}[H]
\centering\footnotesize
\setlength{\abovecaptionskip}{2pt}
\caption{\textbf{I2V evaluation.} VQ:MQ $=1{:}1$.}
\label{tab:i2v_eval}
\setlength{\tabcolsep}{4pt}
\begin{tabular}{lrr}
\toprule
Method & VQ $\uparrow$ & MQ $\uparrow$ \\
\midrule
Flow-GRPO-Fast & 5.771 & 4.710 \\
DanceGRPO & 5.123 & 3.266 \\
DiffusionNFT (EMA) & 5.516 & 4.161 \\
\midrule
\textbf{DiffusionART} & \textbf{11.196} & \textbf{10.476} \\
\bottomrule
\end{tabular}
\end{table}

\subsection{Multi-Reward Optimization for T2I}
\label{app:t2i_multi}
We jointly optimize GenEval, PickScore, CLIPScore, and HPSv2.1 on a single SD3.5-M policy for 500 training steps. Figure~\ref{fig:t2i_scaling_curves} reports each reward separately, comparing how the methods balance compositional accuracy, image--text alignment, and human preference under a shared multi-reward objective.

\begin{figure}[!htbp]
\centering
\setlength{\abovecaptionskip}{2pt}
\setlength{\belowcaptionskip}{0pt}
\includegraphics[width=\linewidth]{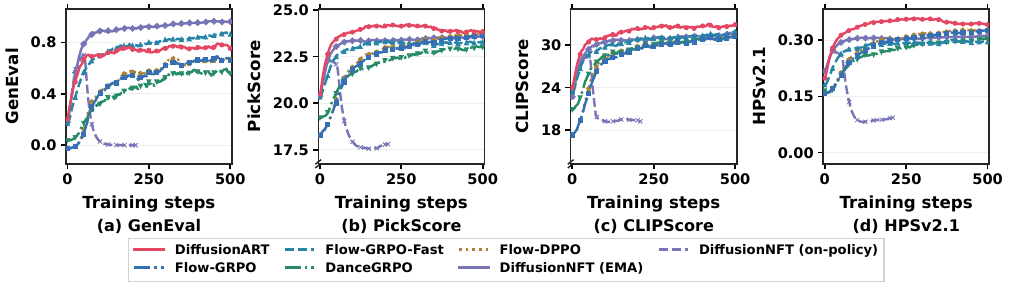}
\vspace{0pt}
\caption{\protect\centering Training reward curves for four jointly optimized objectives on SD3.5-M over 500 steps. PickScore and CLIPScore are shown as unscaled model logits.}
\label{fig:t2i_scaling_curves}
\end{figure}

\subsection{Alternative Reward Weights for T2V}
\label{app:t2v_matched_rewards}

\begin{figure}[!htbp]
\centering
\includegraphics[width=\linewidth]{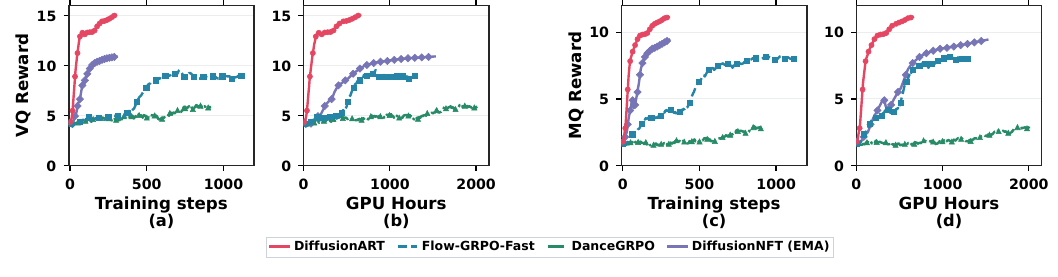}
\setlength{\abovecaptionskip}{2pt}
\setlength{\belowcaptionskip}{0pt}
\caption{\textbf{T2V with matched VQ:MQ $=1{:}1$ weights.} The left pair shows VQ and the right pair MQ, each versus training steps and GPU hours.}
\label{fig:t2v_matched_reward_controls}
\end{figure}

Figure~\ref{fig:t2v_matched_reward_controls} compares all methods with VQ:MQ $=1{:}1$ weights and no TA reward. DiffusionART reaches a training VQ of 6 in 20 steps, compared with 50 for DiffusionNFT (EMA) and 427 for Flow-GRPO-Fast: \textbf{2.5$\times$ and 21.4$\times$ fewer training steps}. It uses 44.2 GPU hours, versus 250.2 and 502.5, giving \textbf{5.7$\times$ and 11.4$\times$ GPU-hour speedups}. DanceGRPO reaches this level after 821 steps and 1805.5 GPU hours. Using the mean raw VQ/MQ reward, $(\mathrm{VQ}+\mathrm{MQ})/2$, DiffusionART reaches the late-stage reward levels of DiffusionNFT (EMA) and Flow-GRPO-Fast with $12.02\times$ and $11.72\times$ fewer GPU hours, respectively (Figure~\ref{fig:teaser}). Its mean reward over the final 50 training steps is 13.018, compared with 10.105 and 8.510 for the two baselines, representing improvements of 28.8\% and 53.0\%.

For MQ of 3, DiffusionART needs 19 steps and 42.9 GPU hours, compared with 33 steps and 165.6 GPU hours for NFT, and 117 steps and 149.4 GPU hours for Fast. This gives \textbf{1.7$\times$/6.2$\times$ step speedups} and \textbf{3.9$\times$/3.5$\times$ GPU-hour speedups} over NFT/Fast. Thus, DiffusionART also learns efficiently when the objective focuses only on VQ and MQ.

\medskip
\noindent\textbf{Direct optimization with a differentiable reward.}
We originally followed the undifferentiable reward model setting (i.e. we can only access the scalar reward value, not including the gradient of the reward model). However, methods like ReFL ~\citep{xu2023imagereward} utilize the gradient of a VLM reward model to help training. We also compare DiffusionART with ReFL under matched VQ:MQ $=1{:}1$ weights (Figure~\ref{fig:t2v_refl_matched}). ReFL uses 32 prompts and one video per prompt ($K=1$), with gradients through the frozen reward model and VAE to the final denoising step; earlier steps are detached. It is trained for 400 updates at a learning rate of $10^{-5}$, without KL regularization or EMA. 

DiffusionART improves both rewards with fewer optimizer updates, while ReFL uses fewer samples per update ($K=1$ versus $K=8$ of DiffusionART) but requires reward-model backpropagation; DiffusionART uses reward values alone.

\begin{figure}[!htbp]
\centering
\includegraphics[width=\linewidth]{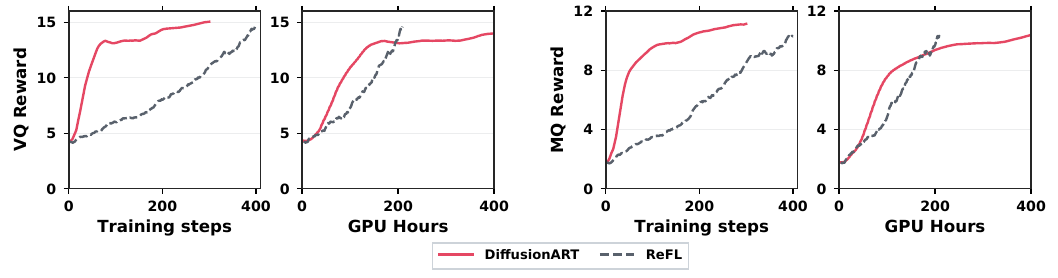}
\setlength{\abovecaptionskip}{2pt}
\caption{\textbf{DiffusionART versus ReFL with matched VQ:MQ $=1{:}1$ weights.} The left pair shows VQ and the right pair MQ, each versus training steps and GPU hours.}
\label{fig:t2v_refl_matched}
\end{figure}

\subsection{Trust-Region Masking for Image Generation}
\label{app:mask_image}
We remove only the trust-region mask, retaining normalized covariance guidance, timestep weighting, and two disjoint MSE updates. On the three SD3.5-M benchmarks and FLUX.1-dev, training remains effective without masking and can improve slightly faster (Figure~\ref{fig:mask_image}). On FLUX.1-dev, the no-mask run reaches a training GenEval score of 0.971 by step 180. The benefit of masking is more pronounced for video generation (Figure~\ref{fig:mask_ablation}); Appendix~\ref{app:mask_diagnostics} examines its stabilizing effect beyond reward curves.

\begin{figure}[!htbp]
\centering
\includegraphics[width=\linewidth]{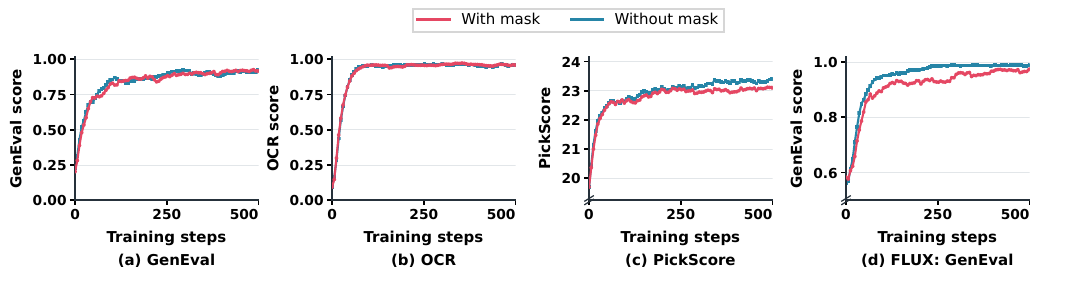}
\setlength{\abovecaptionskip}{2pt}
\caption{\textbf{Image trust-region mask ablation.} (a--c) SD3.5-M; (d) FLUX.1-dev. }
\label{fig:mask_image}
\end{figure}

\subsection{Trust-Region Masking Diagnostics}
\label{app:mask_diagnostics}
We further inspect the matched HunyuanVideo comparison with VQ:MQ weights of
$1{:}1$. Without masking, training exhibits larger late-stage gradient and
velocity-drift spikes (Figure~\ref{fig:mask_diagnostics}), complementing the
reward curves in Figure~\ref{fig:mask_ablation}(a--b).

Although the mask constrains only the second update of each rollout, its stabilizing effect extends to subsequent first updates. Over the shared first 300 training steps, the maximum first-update gradient norm is 1.96 with masking, compared with 3.33 without it. The velocity drift measured before second updates, which reflects displacement caused by the first update, also shows much larger spikes without masking (Figure~\ref{fig:mask_diagnostics}). With masking, the velocity drift score stays a healthy level. These observations suggest that restricting second updates helps keep later first updates well behaved, rather than merely suppressing gradients at the masked update.

\begin{figure}[!htbp]
\centering
\includegraphics[width=\linewidth]{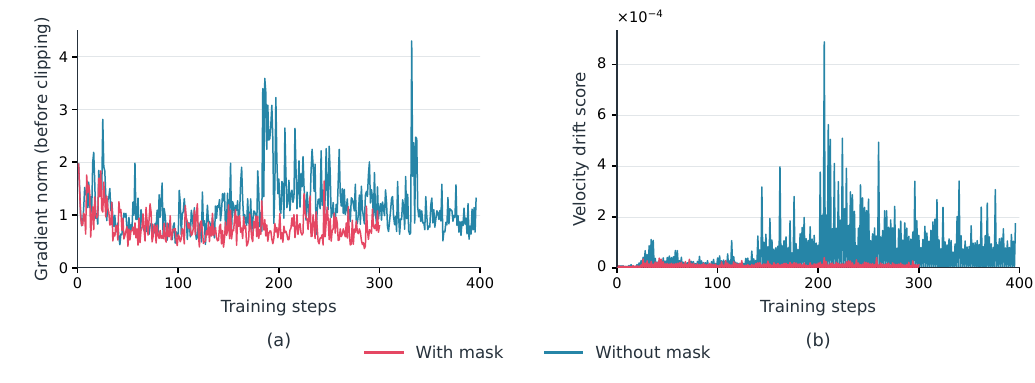}
\setlength{\abovecaptionskip}{2pt}
\caption{\textbf{Trust-region diagnostics on HunyuanVideo-13B.}
(a) Parameter gradient norm before clipping.
(b) Scheduler-reference-weighted velocity drift energy, averaged over training
states before each update.}
\label{fig:mask_diagnostics}
\end{figure}

\FloatBarrier
\endgroup
\section{Proofs}
\label{app:proofs}

\subsection{Conditional Covariance Identity for Ideal NFT}
\label{app:nft_covariance_proof}

\noindent\textbf{Lemma~\ref{lem:nft_covariance} (restated).}
At a fixed supported state $s_t$, let $Q_{\mathrm{old}}$ denote $Q_b$ with $b=\mathrm{old}$, i.e., the joint law obtained by forward noising the rollout endpoints $x_0\sim p_{\mathrm{old}}(\cdot\mid c)$, and let
$q:=\E_{Q_{\mathrm{old}}}[r(x_0,c)\mid s_t]$ and
$g:=\operatorname{Cov}_{Q_{\mathrm{old}}}
[r(x_0,c),\allowbreak u_t\mid s_t]$. If the tilted distributions and moments exist
and $0<q<1$, then
\begin{align*}
q\bigl(v_{\mathrm{fwd}}^+-v_{\mathrm{fwd}}^{\mathrm{old}}\bigr)
&=(1-q)\bigl(v_{\mathrm{fwd}}^{\mathrm{old}}
-v_{\mathrm{fwd}}^-\bigr)=g,
\end{align*}
where all fields are evaluated at $s_t$. Hence the positive and negative
branches point along the same reward-tilt direction $g$.

\noindent\textit{Proof.} Let
$Z_+(c)=\E_{P_{\mathrm{old}}}[r\mid c]>0$ and
$Z_-(c)=\E_{P_{\mathrm{old}}}[1-r\mid c]>0$ normalize $P_+$ and $P_-$. Since $P_\pm$ reweight only $x_0$, while $\epsilon$ is independent of $x_0$ and $(x_t,u_t)$ are functions of $(x_0,\epsilon)$, the joint laws of $(x_0,\epsilon,x_t,u_t)$ satisfy
\begin{align}
Q_+(\cdot\mid t,c)
&=\frac{r(x_0,c)}{Z_+(c)}Q_{\mathrm{old}}(\cdot\mid t,c),\\
Q_-(\cdot\mid t,c)
&=\frac{1-r(x_0,c)}{Z_-(c)}Q_{\mathrm{old}}(\cdot\mid t,c).
\end{align}
Marginalizing onto $x_t$ and using
$q=\E_{Q_{\mathrm{old}}}[r(x_0,c)\mid s_t]$ gives
\begin{align}
Q_+(x_t\mid t,c)
&=\frac{q}{Z_+(c)}Q_{\mathrm{old}}(x_t\mid t,c),\\
Q_-(x_t\mid t,c)
&=\frac{1-q}{Z_-(c)}Q_{\mathrm{old}}(x_t\mid t,c).
\end{align}
Bayes' rule therefore yields
\begin{align}
Q_+(\cdot\mid s_t)
&=\frac{r(x_0,c)}{q}Q_{\mathrm{old}}(\cdot\mid s_t),\\
Q_-(\cdot\mid s_t)
&=\frac{1-r(x_0,c)}{1-q}Q_{\mathrm{old}}(\cdot\mid s_t).
\end{align}
Taking conditional expectations of $u_t$ gives
\begin{align}
q v_{\mathrm{fwd}}^+(s_t)
&=\E_{Q_{\mathrm{old}}}[r u_t\mid s_t],\\
\bigl(1-q\bigr)v_{\mathrm{fwd}}^-(s_t)
&=\E_{Q_{\mathrm{old}}}[(1-r)u_t\mid s_t].
\end{align}
Since $v_{\mathrm{fwd}}^{\mathrm{old}}(s_t)=
\E_{Q_{\mathrm{old}}}[u_t\mid s_t]$,
\begin{align}
q\bigl(v_{\mathrm{fwd}}^+(s_t)-v_{\mathrm{fwd}}^{\mathrm{old}}(s_t)\bigr)
&=\E_{Q_{\mathrm{old}}}[r u_t\mid s_t]
-q\E_{Q_{\mathrm{old}}}[u_t\mid s_t]
=g,\\
\bigl(1-q\bigr)
\bigl(v_{\mathrm{fwd}}^{\mathrm{old}}(s_t)-v_{\mathrm{fwd}}^-(s_t)\bigr)
&=\bigl(1-q\bigr)\E_{Q_{\mathrm{old}}}[u_t\mid s_t]
-\E_{Q_{\mathrm{old}}}[(1-r)u_t\mid s_t]
=g.
\end{align}
This proves both the collinearity and the conditional covariance identity.

\subsection{NFT Force and Restoration Decompositions}
\label{app:nft_decomposition_proofs}

\subsubsection{Equivalent MSE Objective and On-Policy Descent Direction}
\label{app:nft_mse_equivalence}
\noindent\textbf{Preliminary claims (restated).}
For $\beta>0$, freeze $\mathcal D_{\mathrm{old}}$ and treat $r$, $u_t$, and $v_{\mathrm{old}}$ as independent of $\theta$. With $v_{\mathrm{tgt}}=v_{\mathrm{old}}+\frac{2r-1}{\beta}(u_t-v_{\mathrm{old}})$, the NFT objective satisfies
\begin{equation*}
\mathcal L_{\mathrm{NFT}}=\beta^2\mathbb E_{\mathcal D_{\mathrm{old}}}\!\left[\lVert v_\theta-v_{\mathrm{tgt}}\rVert_2^2\right]+\mathrm{const},
\end{equation*}
where $\mathrm{const}$ is independent of $\theta$. At the on-policy initialization $v_\theta=v_{\mathrm{old}}$, the single-sample output-space descent direction is $-\nabla_{v_\theta}\ell_{\mathrm{NFT}}=2\beta(2r-1)(u_t-v_{\mathrm{old}})$.

\begin{proof}
Let $\ell_{\mathrm{NFT}}$ denote the weighted positive--negative loss for one sample, so $\mathcal L_{\mathrm{NFT}}=\mathbb E_{\mathcal D_{\mathrm{old}}}[\ell_{\mathrm{NFT}}]$. Write $d=v_\theta-v_{\mathrm{old}}$, $h=u_t-v_{\mathrm{old}}$, and $A=2r-1$. Expanding and completing the square gives
\begin{align*}
\ell_{\mathrm{NFT}}
&=r\lVert\beta d-h\rVert_2^2+(1-r)\lVert-\beta d-h\rVert_2^2\\
&=\beta^2\lVert d\rVert_2^2-2\beta A\langle d,h\rangle+\lVert h\rVert_2^2\\
&=\beta^2\left\lVert d-\frac{A}{\beta}h\right\rVert_2^2+(1-A^2)\lVert h\rVert_2^2.
\end{align*}
Since $v_{\mathrm{tgt}}=v_{\mathrm{old}}+(A/\beta)h$, the first term is $\beta^2\lVert v_\theta-v_{\mathrm{tgt}}\rVert_2^2$. Using $1-A^2=4r(1-r)$ and taking the expectation over the frozen $\mathcal D_{\mathrm{old}}$ proves Equation~\ref{eq:nft_objective}, with $\mathrm{const}=\mathbb E_{\mathcal D_{\mathrm{old}}}[4r(1-r)\lVert u_t-v_{\mathrm{old}}\rVert_2^2]$ independent of $\theta$.

Differentiating the same sample loss with respect to $v_\theta$ yields
\begin{equation*}
-\nabla_{v_\theta}\ell_{\mathrm{NFT}}=2\beta A h-2\beta^2d.
\end{equation*}
At the on-policy initialization $v_\theta=v_{\mathrm{old}}$, we have $d=0$, so the output-space descent direction is $2\beta(2r-1)(u_t-v_{\mathrm{old}})$, proving the stated proportionality.
\end{proof}

\subsubsection{Reward Guidance and Self-Consistency Force}
\noindent\textbf{Proposition~\ref{prop:nft_contamination} (restated).}
Let $\mathcal L_{\mathrm{NFT}}(s_t)$ denote
Equation~\ref{eq:nft_objective} conditioned on $s_t$, $q$ and $g$ follow Lemma~\ref{lem:nft_covariance}, and $e=v_{\mathrm{fwd}}^{\mathrm{old}}-v_{\mathrm{rvs}}^{\mathrm{old}}$. At the pure on-policy
update point $v_\theta=v_{\mathrm{rvs}}^{\mathrm{old}}$ and for $\beta>0$,
its normalized output-space force is
\begin{equation*}
\mathcal F_{\mathrm{NFT}}(s_t)
:=-\frac{1}{2\beta}
\left.\nabla_{v_\theta(s_t)}\mathcal L_{\mathrm{NFT}}(s_t)
\right|_{v_\theta=v_{\mathrm{rvs}}^{\mathrm{old}}}
=2g+(2q-1)e.
\end{equation*}

\noindent\textit{Proof.} 
Fix $s_t$, and write
$d=v_\theta(s_t)-v_{\mathrm{rvs}}^{\mathrm{old}}(s_t)$ and
$h=u_t-v_{\mathrm{rvs}}^{\mathrm{old}}(s_t)$. The conditional objective is
\begin{align}
\mathcal L_{\mathrm{NFT}}(s_t)
&=\E_{Q_{\mathrm{old}}}\!\left[r\lVert\beta d-h\rVert_2^2
 +(1-r)\lVert-\beta d-h\rVert_2^2\mid s_t\right]\\
&=\beta^2\lVert d\rVert_2^2
-2\beta\left\langle d,
\E_{Q_{\mathrm{old}}}[(2r-1)h\mid s_t]\right\rangle
+\E_{Q_{\mathrm{old}}}[\lVert h\rVert_2^2\mid s_t].
\end{align}
At the pure on-policy update point, $d=0$. Therefore,
\begin{align}
\mathcal F_{\mathrm{NFT}}(s_t)
&=-\frac{1}{2\beta}
\left.\nabla_d\mathcal L_{\mathrm{NFT}}(s_t)\right|_{d=0}\\
&=\E_{Q_{\mathrm{old}}}[(2r-1)h\mid s_t].
\end{align}
Because $\E_{Q_{\mathrm{old}}}[h\mid s_t]
=v_{\mathrm{fwd}}^{\mathrm{old}}-v_{\mathrm{rvs}}^{\mathrm{old}}=e$ and
$\operatorname{Cov}_{Q_{\mathrm{old}}}(r,h\mid s_t)
=\operatorname{Cov}_{Q_{\mathrm{old}}}(r,u_t\mid s_t)=g$,
\begin{align}
\E_{Q_{\mathrm{old}}}[(2r-1)h\mid s_t]
&=2\operatorname{Cov}_{Q_{\mathrm{old}}}(r,h\mid s_t)
 +(2\E_{Q_{\mathrm{old}}}[r\mid s_t]-1)
 \E_{Q_{\mathrm{old}}}[h\mid s_t]\\
&=2g+(2q-1)e,
\end{align}
which proves Equation~\ref{eq:nft_force_decomposition}.

\subsubsection{Guidance and Quadratic Restoration}
\noindent\textbf{Proposition~\ref{prop:nft_guidance_restore} (restated).} Let $d=v_\theta-v_{\mathrm{old}}$ and $h=u_t-v_{\mathrm{old}}$.
\begin{equation*}
\mathcal L_{\mathrm{NFT}}
=\E_{\mathcal D_{\mathrm{old}}}\!\left[-2\beta(2r-1)
\langle d,h\rangle
+\beta^2\lVert d\rVert_2^2\right]+\mathrm{const},
\end{equation*}
where $\mathrm{const}$ is independent of the trainable parameters.

\noindent\textit{Proof.} 
Using the same $d$ and $h$, the positive and negative residuals are
$\beta d-h$ and $-\beta d-h$, respectively. Expanding their weighted squared
norms gives
\begin{align}
r\lVert\beta d-h\rVert_2^2
+(1-r)\lVert-\beta d-h\rVert_2^2
=\beta^2\lVert d\rVert_2^2
-2\beta(2r-1)\langle d,h\rangle
+\lVert h\rVert_2^2.
\end{align}
Taking $\E_{\mathcal D_{\mathrm{old}}}$ and absorbing the final term into the constant proves
Equation~\ref{eq:nft_guidance_restore_decomposition}.

\subsection{Covariance as the First-Order Reward Tilt}
\label{app:art_details}
\noindent\textbf{Lemma~\ref{lem:art_tilt} (restated).}
If the relevant moments exist and differentiation can be
exchanged with expectation near $\tau=0$, then
$\left.\frac{\partial}{\partial\tau}\E_{p_\tau}[x_0\mid c]\right|_{\tau=0}
=\operatorname{Cov}_{p_{\mathrm{old}}}(R,x_0\mid c)$.

\noindent\textit{Proof.} 
Fix $c$ and suppress it below. Write
$M(\tau)=\E_{p_{\mathrm{old}}}[x_0e^{\tau R}]$ and
$Z(\tau)=\E_{p_{\mathrm{old}}}[e^{\tau R}]$.
The normalized tilted density is $p_\tau(x_0)=p_{\mathrm{old}}(x_0)e^{\tau R(x_0)}/Z(\tau)$,
so its endpoint mean is $M(\tau)/Z(\tau)$.
Assume differentiation under these expectations is valid; a sufficient
condition is $\E[(1+\lVert x_0\rVert)(1+|R|)e^{a|R|}]<\infty$
for some $a>0$. Since $Z(0)=1$, the quotient rule gives
\begin{equation}
\left.\frac{\partial}{\partial\tau}\frac{M(\tau)}{Z(\tau)}\right|_0
=\E[x_0R]-\E[x_0]\E[R]
=\operatorname{Cov}(R,x_0).
\end{equation}

\medskip
\noindent\textbf{From shared guidance to statewise correction.}
\label{app:art_guidance_approximation}
Lemma~\ref{lem:art_tilt} motivates a shared reward-improving direction computed from the $K$ endpoints already sampled per prompt.
This construction avoids state-conditional posterior sampling and excludes the NFT self-consistency term by construction.
However, a prompt-level direction need not match the posterior guidance at every state, and different states can respond differently to the same parameter update.
Without statewise correction, the resulting deviations can accumulate, providing a explanation for the gradual deterioration with a single update in Figure~\ref{fig:t2i_ablation_curves}.
The second MSE update complements shared guidance with local feedback: its restoring force grows with each state's displacement from the rollout prediction.
This controls drift without requiring posterior sampling, combining scalable guidance with statewise stabilization.

\section{Clarification of Methods and Experiments}
\label{app:clarifications}

\subsection{Reverse-Process Policy Optimization}
\label{app:flowgrpo_objective}

FlowGRPO treats denoising as a finite-horizon policy. For each prompt, it samples
$K$ stochastic reverse trajectories, group-normalizes their terminal rewards,
and applies clipped PPO to every transition~\citep{liu2025flowgrpo}. This
directly supervises sampled states but requires tractable transition densities
and stored trajectories.

For reverse trajectory
$x_{T:0}^{(i)}\sim P_{\mathrm{old}}^{\mathrm{rev}}(\cdot\mid c)$,
FlowGRPO assigns the group-normalized
terminal advantage $A_i$ to every transition and maximizes
\begin{equation}
\begin{aligned}
\mathcal J_{\mathrm{FG}}
=\E_{i,k}\!\Big[&
\min\!\big(\rho_{i,k}A_i,\,
\operatorname{clip}(\rho_{i,k},1-\epsilon_{\mathrm{clip}},
1+\epsilon_{\mathrm{clip}})A_i\big)\\
&-\lambda_{\mathrm{KL}}D_{\mathrm{KL}}\!\big(
p_\theta(\cdot\mid x_k^{(i)},c)\,\|\,
p_{\mathrm{ref}}(\cdot\mid x_k^{(i)},c)\big)\Big].
\end{aligned}
\end{equation}
where $\rho_{i,k}=p_\theta(x_{k-1}^{(i)}\mid x_k^{(i)},c)/
p_{\mathrm{old}}(x_{k-1}^{(i)}\mid x_k^{(i)},c)$ and
$\E_{i,k}$ also averages over prompts and old-policy trajectories.
Here $\epsilon_{\mathrm{clip}}>0$ is the PPO clip radius,
$\lambda_{\mathrm{KL}}\geq0$ weights the penalty, and $p_{\mathrm{ref}}$ is a
fixed reference transition. This requires the stochastic
reverse transition used for rollout to expose a tractable density.

\subsection{Reference Transitions for ODE Trust-Region Masking}
\label{app:art_mask}

We strictly follow Flow-GRPO's ODE-to-SDE construction~\citep{liu2025flowgrpo}; we refer to that work for the derivation. With reverse step size $h=t-t_{\rm next}>0$ and state $s_t = (x_t, t, c)$, where $x_t$ is the current latent, the reference transition has mean and variance when $1>t>0$:
\begin{equation}
 \begin{aligned}
 \mu(s_t,v)
 &=x_t+h\left[\left(1+\frac{\nu^2}{2}\right)v-\frac{\nu^2}{2(1-t)}x_t\right],\\
 \sigma_{\rm ref}^2(t)&=\nu^2\frac{t}{1-t}h.
 \end{aligned}
 \label{eq:art_reference_transition}
\end{equation}
Substituting $v_\theta(s_t)$ and $v_\theta^{(0)}(s_t)$ gives the current and old reference means, respectively, hence $\mu(s_t,v_\theta(s_t))-\mu(s_t,v_\theta^{(0)}(s_t))=h(1+\nu^2/2)d(s_t)$. At $t=1$, we use the endpoint convention $\mu(s_t,v)=x_t+hv$ and $\sigma_{\rm ref}^2(1)=\nu^2$.

We use $\nu=0.8$ and $\delta=10^{-5}$ without further tuning. The mask is recomputed before each update and treated as constant during backpropagation.

\subsection{Pairwise Feedback: Identity and Implementation}
\label{app:art_pairwise}
\noindent\textbf{Pairwise covariance identity.}
\label{cor:art_pairwise}
For $K$ endpoints forming $M=K/2$ disjoint pairs, with pairwise labels $y_i\in\{0,\tfrac12,1\}$ replacing $R_i$ in Equation~\ref{eq:art_empirical_cov}, use population reduction $1/K$ globally and $1/2$ within each pair. The global
covariance equals the average of the $M$ pair covariances:
$\overline{\operatorname{Cov}}=(2K)^{-1}\sum_{\substack{j=1\\j\,\mathrm{non\text{-}tied}}}^{M}
(x_{0,j}^{\mathrm{win}}-x_{0,j}^{\mathrm{lose}})$.
Equivalently, the coefficient is $1/(4M)$. Tied pairs contribute zero. For $K=2$, the covariance reduces to one quarter of the winner-minus-loser difference: after normalization, the guidance simply points from the worse sample to the better one. A tie yields zero guidance.

\noindent\textit{Proof.} Every pair has label mean $1/2$, including ties, so the global label mean is
also $1/2$. Thus centering endpoints changes no cross term:
\begin{align}
\overline{\operatorname{Cov}}&=\frac1K\sum_{i=1}^{K}(y_i-\tfrac12)(x_0^{(i)}-\bar x_{0,c})
=\frac1K\sum_{i=1}^{K}(y_i-\tfrac12)x_0^{(i)}\\
&=\frac1{2K}\sum_{\substack{j=1\\j\,\mathrm{non\text{-}tied}}}^{M}
\bigl(x_{0,j}^{\mathrm{win}}-x_{0,j}^{\mathrm{lose}}\bigr).
\end{align}
A non-tied pair's covariance, computed with denominator two, is one quarter
of its winner-minus-loser difference. Averaging those pair covariances gives
the same expression; tied pairs contribute zero. The pairwise labels enter
the same covariance operator as scalar rewards, establishing the stated
specialization. The coefficient $\frac{1}{2K}$ doesn't matter since the covariance will be normalized to form the guidance.

\par\addvspace{8pt}
\noindent\begin{minipage}[t]{0.55\linewidth}
\vspace{0pt}
\textbf{Guidance assignment.}\label{app:pairwise_assignment}
We compare prompt-global and pair-local guidance on GenEval, using 24 images per prompt to form 12 disjoint pairs. Prompt-global guidance aggregates pair covariances before normalization and applies the shared direction to all trajectories. Pair-local guidance normalizes each pair's covariance separately and applies it only to that pair. At step 400, prompt-global guidance reaches a held-out GenEval score of \textbf{0.91}, versus \textbf{0.78} for pair-local guidance. It also achieves higher training rewards (Figure~\ref{fig:t2i_pair_assignment_appendix}), showing that the better practice is pooling the pairs of the same prompt together before computing the covariance to reduce the variance.
\end{minipage}\hfill
\begin{minipage}[t]{0.41\linewidth}
\vspace{0pt}
\centering
\includegraphics[width=\linewidth]{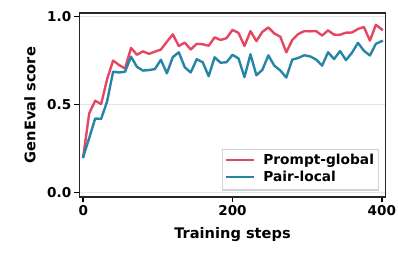}
\setlength{\abovecaptionskip}{2pt}
\captionof{figure}{Guidance assignment on GenEval.}
\label{fig:t2i_pair_assignment_appendix}
\end{minipage}
\par\addvspace{8pt}

\begin{algorithm}[H]
\caption{DiffusionART-Pairwise: one training iteration}\label{alg:art_pairwise}
\small
\begin{algorithmic}[1]
\State Sample with an ODE solver using $v_\theta^{(0)}=v_\theta$. \hfill\textcolor{black!50}{\textbf{Rollout Stage}}
\Statex \hspace{\algorithmicindent}For each prompt $c$, collect an even number $K$ of endpoints $x_0^{(i)}$, states $s_t$, and predictions $v_\theta^{(0)}(s_t)$.
\State Form $M=K/2$ disjoint pairs per prompt; query the pairwise reward model.
\State Assign pairwise rewards $y_i=1$ to winners, $y_i=0$ to losers, and $y_i=\tfrac12$ for ties.
\State Compute covariance guidance by Equation~\ref{eq:art_empirical_cov}, replacing scalar rewards $R_i$ with pairwise rewards $y_i$.

\State Freeze $v_{\mathrm{tgt}}(s_t)=v_\theta^{(0)}(s_t)+G(c)$ for each state.
\Statex \leavevmode\leaders\hbox{\rule[0.5ex]{2pt}{0.3pt}\hskip2pt}\hfill\mbox{}
\State Split the samples into two equal-sized disjoint minibatches $B_1,B_2$. \hfill\textcolor{black!50}{\textbf{Training Stage}}
\State For each $B\in\{B_1,B_2\}$, recompute $w(s_t)$ and update $\theta$ with one step on
\Statex \hspace{\algorithmicindent}$\mathcal L(B;\theta)=\dfrac{1}{|B|}\sum_{s_t\in B}w(s_t)\operatorname{MSE}(v_\theta(s_t),v_{\mathrm{tgt}}(s_t))$.

\end{algorithmic}
\end{algorithm}

\clearpage
\section{Visualization}
\label{app:visualization}
\begin{figure}[H]
\centering
\includegraphics[width=\linewidth]{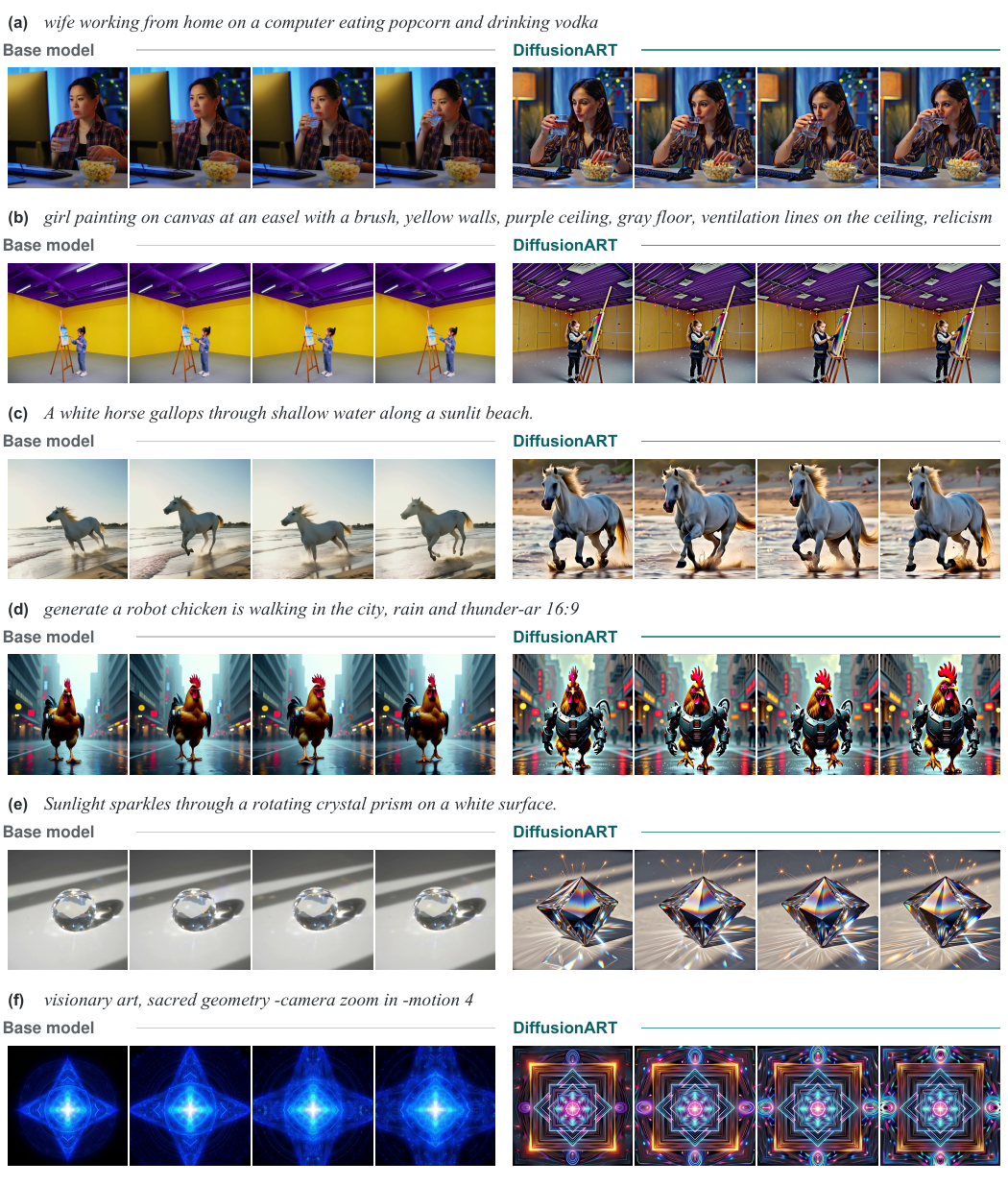}
\caption{\textbf{Text-to-video generation before and after DiffusionART.} Each pair uses the same prompt, seed, and sampling settings. Four chronologically ordered frames are shown at identical frame positions within each pair: pretrained outputs on the left and DiffusionART outputs on the right.}
\label{fig:t2v_qualitative}
\end{figure}
\clearpage

\clearpage
\begin{figure}[H]
\centering
\includegraphics[width=\linewidth]{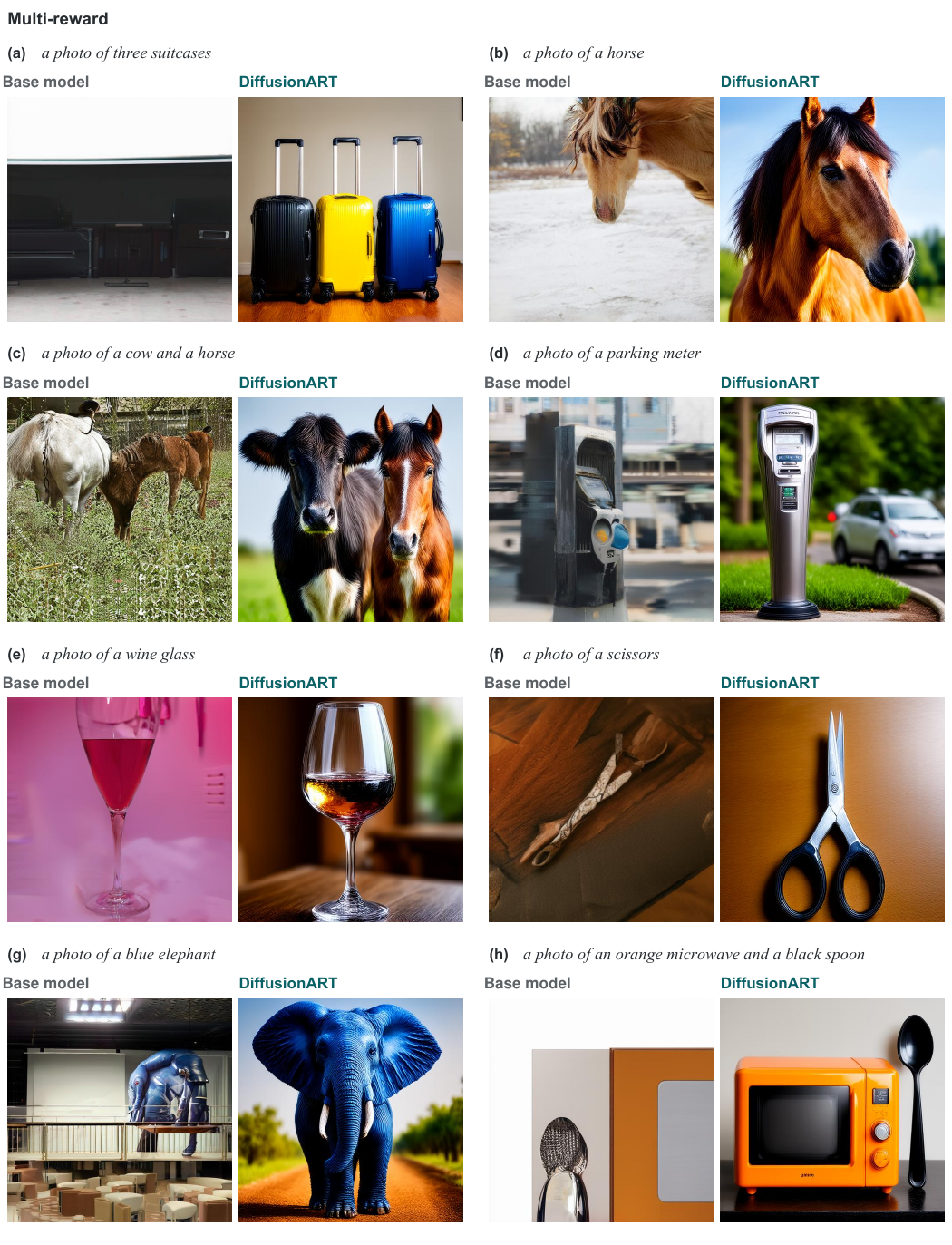}
\caption{\textbf{Text-to-image generation with multi-reward optimization.} Selected SD3.5-M outputs before and after DiffusionART, jointly optimized with GenEval, CLIPScore, PickScore, and HPSv2. Each pair shows the same prompt, with pretrained outputs on the left and DiffusionART outputs on the right.}
\label{fig:t2i_qualitative_multi}
\end{figure}
\clearpage
\begin{figure}[H]
\centering
\includegraphics[width=\linewidth]{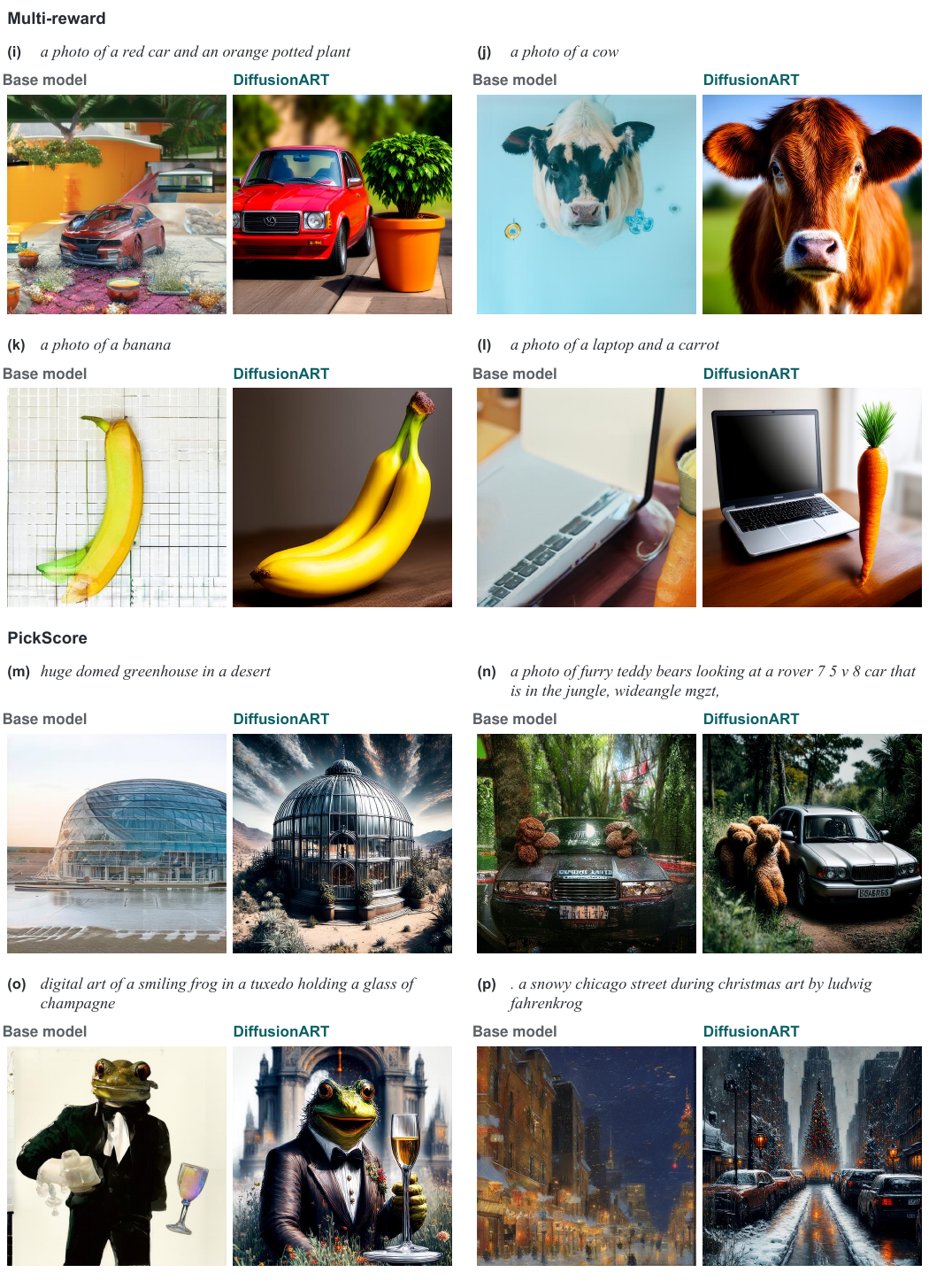}
\caption{\textbf{Additional text-to-image comparisons.} The top two rows use multi-reward optimization; the bottom two use PickScore. Each pair shows pretrained and DiffusionART outputs under the same prompt.}
\label{fig:t2i_qualitative_pick}
\end{figure}
\FloatBarrier
\clearpage

\end{document}